\documentclass[11pt]{article}

\usepackage{fullpage,times}

\usepackage[utf8]{inputenc}
\usepackage[T1]{fontenc}
 \usepackage{hyperref}
\hypersetup{
        colorlinks   = true,
        urlcolor     = blue,
        linkcolor    = blue,
        citecolor   = blue
}

\usepackage{url}
\usepackage{booktabs}
\usepackage{amsmath,amsfonts,amsthm,amssymb}
\usepackage{nicefrac}
\usepackage{microtype}
\usepackage{natbib}

\usepackage{enumitem}
\usepackage{algorithm}
\usepackage{algorithmicx}
\usepackage{algpseudocode}
\usepackage[nolist,nohyperlinks]{acronym}
\usepackage{color}
\usepackage{bold-extra}
\usepackage{mathtools}
\usepackage{etoolbox}
\usepackage{xfrac}
\usepackage[capitalize]{cleveref}
\usepackage[caption=false]{subfig}
\usepackage{graphbox}
\usepackage{wrapfig}

\newcommand{\bhatM}{\boldsymbol{\widehat{M}}}

\newcommand{\pinv}{\dagger}

\newcommand{\bA}{\boldsymbol{A}}

\newcommand{\bC}{\boldsymbol{C}}

\newcommand{\bx}{\boldsymbol{x}}

\newcommand{\ba}{\boldsymbol{a}}

\newcommand{\bX}{\boldsymbol{X}}

\newcommand{\bM}{\boldsymbol{M}}

\newcommand{\bhG}{\boldsymbol{\wh{G}}}

\newcommand{\bI}{\boldsymbol{I}}
\newcommand{\bu}{\boldsymbol{u}}
\newcommand{\by}{\boldsymbol{y}}

\newcommand{\sS}{\mathcal{S}}

\newcommand{\bz}{\boldsymbol{z}}

\newcommand{\bS}{\boldsymbol{S}}

\newcommand{\bw}{\boldsymbol{w}}
\newcommand{\btilw}{\boldsymbol{\tilde{w}}}
\newcommand{\btilwstar}{\tilde{\bw}^{\star}}

\newcommand{\bW}{\boldsymbol{W}}

\newcommand{\bV}{\boldsymbol{V}}
\newcommand{\bU}{\boldsymbol{U}}
\newcommand{\bv}{\boldsymbol{v}}
\newcommand{\bzero}{\boldsymbol{0}}
\newcommand{\balpha}{\boldsymbol{\alpha}}
\newcommand{\sA}{\mathcal{A}}

\newcommand{\sN}{\mathcal{N}}
\newcommand{\sM}{\mathcal{M}}
\newcommand{\sB}{\mathcal{B}}

\newcommand{\sE}{\mathcal{E}}

\newcommand{\sD}{\mathcal{D}}

\newcommand{\sW}{\mathcal{W}}
\newcommand{\sZ}{\mathcal{Z}}

\DeclareMathOperator*{\argmin}{arg\,min}

\DeclareMathOperator*{\rank}{rank}

\newcommand{\field}[1]{\mathbb{#1}}

\newcommand{\R}{\field{R}}

\newcommand{\sctilO}{\mathcal{\tilde{O}}}

\newcommand{\wh}{\widehat}
\newcommand{\ve}{\varepsilon}
\newcommand{\bve}{\boldsymbol{\varepsilon}}

\newcommand{\reals}{\mathbb{R}}

\newcommand{\tp}{^{\top}}

\DeclareMathOperator{\E}{\mathbb{E}}

\newcommand*\diff{\mathop{}\!\mathrm{d}}
\newcommand{\lmin}{\lambda_{\mathrm{min}}}
\newcommand{\lminh}{\widehat{\lambda}_{\mathrm{min}}}
\newcommand{\lmax}{\lambda_{\mathrm{max}}}

\newcommand{\smin}{s_{\mathrm{min}}}
\newcommand{\smax}{s_{\mathrm{max}}}

\newcommand{\bSigma}{\boldsymbol{\Sigma}}
\newcommand{\bhSigma}{\boldsymbol{\wh{\Sigma}}}
\newcommand{\eps}{\epsilon}

\newcommand{\pr}[1]{\left( #1 \right)}
\newcommand{\br}[1]{\left[ #1 \right]}
\newcommand{\cbr}[1]{\left\{ #1 \right\}}

\newcommand{\lf}{\left}
\newcommand{\rt}{\right}

\newcommand{\bwstar}{\bw^{\star}}

\newcommand{\Lip}{\Delta}

\newcommand{\leqC}{\lesssim}

\newcommand{\alg}{\sA_S}

\newcommand{\algmap}{\sA}
\newcommand{\Lh}{\hat{L}}

\newtheorem{lemma}{Lemma}

\newtheorem{theorem}{Theorem}
\newtheorem{cor}{Corollary}

\newtheorem{prop}{Proposition}
\newtheorem{definition}{Definition}
\newcommand{\bmid}{\;\middle|\;}

\newcommand{\initvar}{\nu^2_{\mathrm{init}}}

\newcommand{\noisevar}{\sigma^2}
\begin{acronym}
\acro{RLS}{Regularized Least Squares}
\acro{ERM}{Empirical Risk Minimization}
\acro{RKHS}{Reproducing kernel Hilbert space}
\acro{DA}{Domain Adaptation}
\acro{PSD}{Positive Semi-Definite}
\acro{SGD}{Stochastic Gradient Descent}
\acro{OGD}{Online Gradient Descent}
\acro{GD}{Gradient Descent}
\acro{SGLD}{Stochastic Gradient Langevin Dynamics}
\acro{IS}{Importance Sampling}
\acro{WIS}{Weighted Importance Sampling}
\acro{MGF}{Moment-Generating Function}
\acro{ES}{Efron-Stein}
\acro{ESS}{Effective Sample Size}
\acro{KL}{Kullback-Liebler}
\acro{SVD}{Singular Value Decomposition}
\acro{NTK}{Neural Tangent Kernel}
\acro{DD}{Double Descent}
\end{acronym}

\title{On the Role of Optimization in Double Descent:\\ A Least Squares Study}
\date{}

\author{
  Ilja Kuzborskij\\
  DeepMind\\
 \and
 Csaba Szepesv\'ari\\
 DeepMind, Canada\\
 University of Alberta, Edmonton\\
 \and
 Omar Rivasplata\\
 University College London\\
 \and
 Amal Rannen-Triki\\
 DeepMind
 \and
 Razvan Pascanu\\
 DeepMind
}

\begin{document}

\maketitle

\begin{abstract}

 Empirically it has been observed that the performance of deep neural networks steadily improves as we increase model size, contradicting the classical view on overfitting and generalization. 
 Recently, the \emph{double descent} phenomena has been proposed to reconcile this observation with theory, suggesting that the test error has a second descent when the model becomes sufficiently overparametrized, as the model size itself acts as an implicit regularizer.
 In this paper we add to the growing body of work in this space, providing a careful study of learning dynamics as a function of model size for the least squares scenario. 
 We show an excess risk bound for the gradient descent solution of the least squares objective. The bound depends on the smallest non-zero eigenvalue of the covariance matrix of the input features, via a functional form that has the double descent behaviour. 
 This gives a new perspective on the double descent curves reported in the literature.
 Our analysis of the excess risk allows to decouple the effect of optimisation and generalisation error.
 In particular, we find that in case of noiseless regression,  double descent 
 is explained solely by  optimisation-related quantities, which was  missed in studies focusing on the Moore-Penrose pseudoinverse solution.
 We believe that our derivation provides an alternative view compared to existing work, shedding some light on a possible cause of this phenomena, at least in the considered least squares setting. 
 We empirically explore if our predictions hold for neural networks, in particular whether the covariance of intermediary hidden activations has a similar behaviour as the one 
 predicted by our derivations. 
 
\end{abstract}

\section{Introduction}
Deep Neural Networks have shown amazing versatility across a large range of domains. Among one of their main features is their ability to perform better with scale. Indeed, some of the most impressive results [see e.g. \citealp{brock2021high,brown2020language,senior2020improved,schrittwieser2020mastering,silver2017mastering,he2016deep} and references therein] have been obtained often by exploiting this fact, 
leading to models that have at least as many parameters as the number of examples in the dataset they are trained on. 
Empirically, the limitation on the model size seems to be mostly imposed by hardware or compute.  From a theoretical point of view, however, this property is quite surprising and counter-intuitive, as one would expect that in such extremely overparametrized regimes the learning
would be prone to overfitting~\citep{hastie2009elements,shalev2014understanding}. 

Recently \cite{belkin2019reconciling} proposed
\ac{DD} phenomena as an explanation. 
They argue that the classical view of overfitting does not apply in extremely over-parameterized regimes, which were less studied prior to the emergence of the deep learning era. 
The classical view in the parametric learning models was based on error curves showing that the training error decreases monotonically when plotted against model size, while the corresponding test errors displayed a U-shape curve, where the model size for the bottom of the U-shape was taken to achieve the ideal trade-off between model size and generalization, and larger model sizes than that were thought to lead to `overfitting' since the gap between test errors and training errors increased.
\begin{figure}[H]
  \centering
  \includegraphics[width=.95\textwidth]{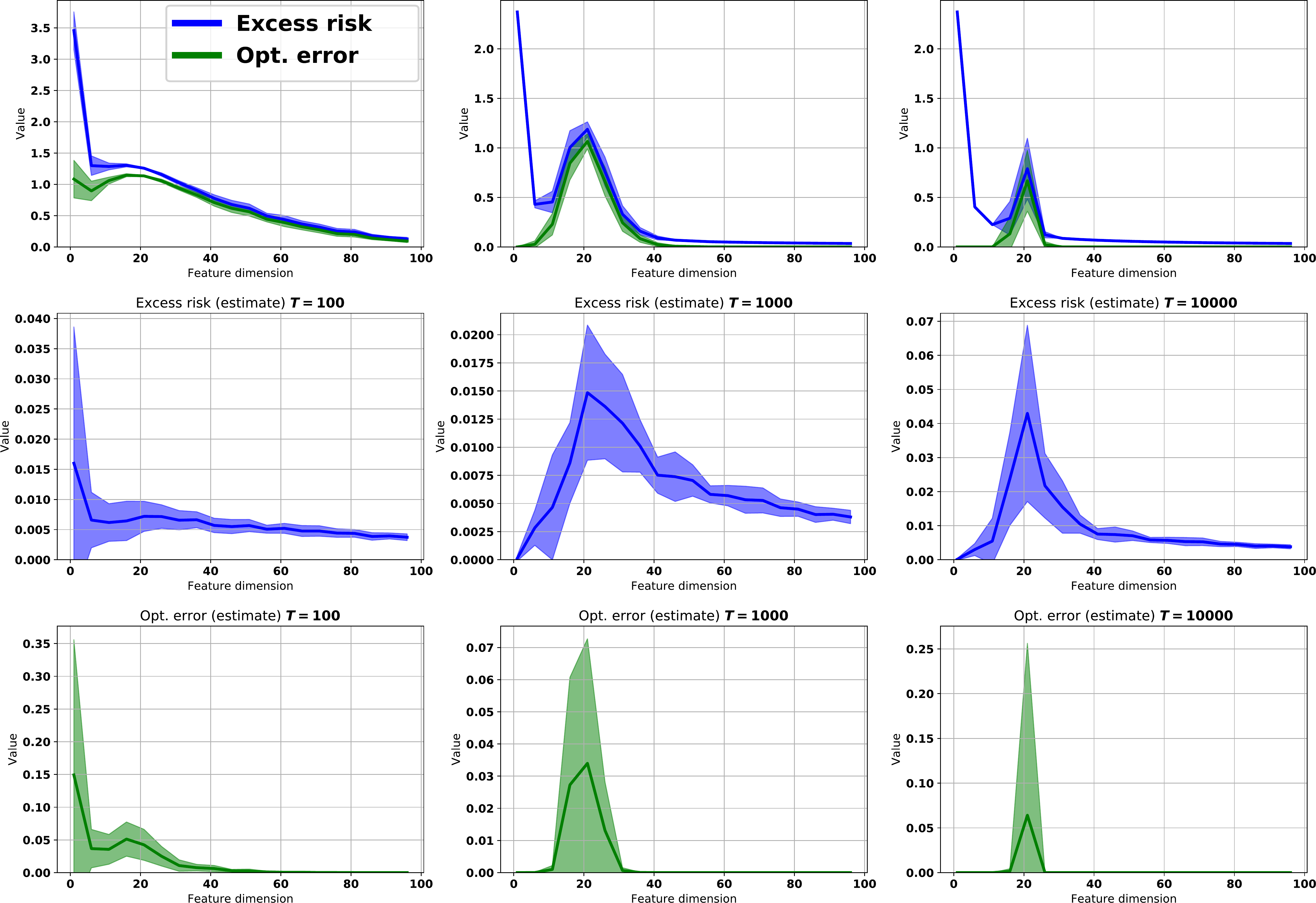}
  \caption{\small
    Evaluation of a synthetic setting inspired by~\cite{belkin2019two}.
    We consider a linear regression problem ($n=20$, $d \in [100]$), where regression parameters are fixed, and instances are sampled from $[-1,1]$-truncated normal density.
    \ac{GD} is run with $\alpha = 0.05$ and initialization variance is set as $\initvar = 1/d$.
    The first row demonstrates behavior of~\eqref{eq:intro:excess_ls}, the second shows an estimate of the excess risk (on $10^4$ held-out points), and the third an estimate of the optimization error. 
    }
  \label{fig:ls_emp_bound_vs_sim}
\end{figure}

The classical U-shape error curve 
dwells in what is now called the under-parameterized regime, where the model size is smaller than the size of the dataset.
Arguably, the restricted model sizes used in the past were
tied to 
the available computing power.
By contrast, it is common nowadays for model sizes to be
larger than the amount of available data, which we call the over-parameterized regime. 
The divide between these two regimes 
is marked by a point where model size matches dataset size, which \cite{belkin2019reconciling} called the \emph{interpolation threshold}.

The work of \cite{belkin2019reconciling} argues that as model size grows beyond the interpolation threshold, one will observe a second descent of the test error that asymptotes in the limit to smaller values than those in the underparameterized regime, which indicates better generalization rather than overfitting.
To some extent this was already known in the \emph{nonparametric} learning where model complexity scales with the amount of data by design (such as in nearest neighbor rules and kernels), yet one can generalize well and even achieve statistical consistency \citep{gyorfi2002distribution}.
This has lead to a growing body of works trying to identify the mechanisms behind DD, to which the current manuscript belongs too. We refer the reader to \cref{sec:related_work}, where the related literature is discussed.
Similar to these works, our goal is also to understand the cause of \ac{DD}.
Our approach
is slightly different:
 we explore the least squares problem that allows us to work with analytic expressions for all the quantities involved. 
 \cref{fig:ls_emp_bound_vs_sim} provides a summary of our findings. 
 In particular, it shows the behaviour
 of the excess risk in a setting with random inputs and noise-free labels, for which in \cref{sec:excess} we prove a bound that has the form $\E\Bigl[(1-\alpha \lminh^+)^{2T}\Bigr] \|\bwstar\|^2 + \frac{\|\bwstar\|^2}{\sqrt{n}}$, for a rapidly decaying spectrum of the sample covariance.
 In this setting, the linear predictors project $d$-dimensional features
 by dot product with a weight vector which must be learned from data; then $\bwstar$ refers to the optimal solution, $\alpha$ is a constant learning rate, and $n$ is the number of examples in the training set. 
 Note that the feature dimension $d$ coincides with the number of parameters in this particular setting, hence $d>n$ is the overparameterized regime.
 The quantity $\lminh^+$ is of special importance: \emph{It is the smallest positive eigenvalue of the sample covariance matrix of the features.}
 In particular, we observe that 
 the excess risk is controlled by the smallest non-zero eigenvalue of the covariance of the features, and its functional dependence exhibits a profile similar to the \ac{DD} curve. This offers a new perspective on the problem.

In~\cref{fig:ls_emp_bound_vs_sim}
we observe \emph{a peaking behavior}, not only in the excess risk, but also in the quantity that we label `optimization error' which is a special term of the excess risk bound that is purely related to optimization.  
The peaking behaviour of the excess risk (MSE in case of the square loss) was observed and studied in a number of settings \citep{belkin2019reconciling,mei2019generalization,derezinski2020exact};
however, the connection between the peaking behavior and optimization so far received less attention.
This pinpoints a less-studied setting and we conjecture that the \ac{DD} phenomenon occurs due to
$\lminh^+$.
In the absence of label noise, we conclude that \ac{DD} manifests due to the optimization process.
On the other hand, when label noise is present, in addition to the optimization effect, $\lminh^+$ also has an effect on the generalization error.

\paragraph{Our contributions:}
Our main theoretical contribution is provided in \cref{sec:excess}. In particular, \cref{sec:ls_random_design} focuses on the noise-free least squares problem,
\cref{sec:ls_random_design_plus_noise} adds noise to the problem, and \cref{sec:lmin_concentration} deals with concentration of the sample-dependent $\lminh^+$ around its population counterpart. Sections~\ref{sec:excess_risk} and~\ref{sec:empirical_discussion} provide an in-depth discussion on the implications of our findings and an empirical exploration of the question whether simple neural networks have a similar behaviour.

\paragraph{Notation:}
The linear algebra/analysis notation used in this work is defined in~\cref{app:sec:defs}.
We briefly mention here that we denote column vectors and matrices with small and capital bold letters, respectively, e.g. $\balpha=[\alpha_1, \alpha_2, \ldots, \alpha_d]\tp \in \R^d~$
and $\bA \in \R^{d_1 \times d_2 }$.
Singular values of a rectangular matrix $\bA \in \R^{n \times d}$ are denoted by $\smax(\bA) = s_1(\bA) \geq \ldots \geq s_{n \wedge d}(\bA) = \smin(\bA)$.
The rank of $\bA$ is $r = \max\{ k \mid s_k(\bA) > 0 \}$.
Eigenvalues of a \ac{PSD} matrix $\bM \in \R^{d \times d}$ are non-negative and are denoted 
$\lmax(\bM) = \lambda_1(\bM) \geq \ldots \geq \lambda_d(\bM) = \lmin(\bM)$,
while the smallest \emph{non-zero} eigenvalue is denoted $\lmin^+(\bM)$.

Next, we set the learning theory notation.
In a parametric statistical learning problem the learner is given a training set $S = \pr{Z_1, \ldots, Z_n}$, which is an $n$-tuple consisting of independent random elements, called training examples, distributed according to some unknown
distribution $\sD \in \sM_1(\sZ)$, where $\sZ$ is called the \emph{example space}.  
The learner's goal is to select parameter $\bw$ from some \emph{parameter space} $\sW$ so as to minimize the \emph{population loss} $L(\bw) = \int_{\sZ} \ell(\bw, z) \sD(\diff z)$, where $\ell : \sW \times \sZ \to [0, 1]$ is some given \emph{loss function}.
A learner following the \ac{ERM}
principle selects a $\bw$ with the smallest \emph{empirical loss} $\Lh_S(\bw) = (\ell(\bw, Z_1) + \dots + \ell(\bw,
Z_n)) / n$ over the training set.
In this report we consider a Euclidean parameter space: $\sW = \reals^d$.

We consider a least squares regression problem.
In this setting, each example is an instance-label pair: $Z_i = (\bX_i,Y_i) \in \sB_1 \times [0,1]$.
We assume that inputs $\bX_i$ are from the Euclidean ball of unit radius $\sB_1 \subset \reals^d$, and labels $Y_i$ are in the unit interval $[0,1]$.
For a suitably chosen parameter vector $\bw$,
the noiseless regression model is $f(\bX) = \bX\tp \bw$ and the model with label noise is $f(\bX) = \bX\tp \bw + \epsilon$ where $\epsilon \sim \sN(0,\sigma^2)$.
The loss function is the square loss: $\ell(\bw, Z_i) = (f(\bX_i) - Y_i)^2 / 2$.

\section{Related Work}
\label{sec:related_work}
The literature on the \ac{DD}
of the test error has mainly focused on the ordinary least squares with the explicit solution given by the Moore-Penrose pseudo-inverse. Early works have focused on  instance-specific settings (making distributional assumptions on the inputs) while arguing when the analytic pseudo-inverse solutions yield DD behaviour~\citep{belkin2019two}.
This was later extended to a more general setting showcasing the control of \ac{DD} by the \emph{spectrum} of the feature matrix~\citep{derezinski2020exact}.
In this paper we also argue that the spectrum of the covariance matrix has a critical role in \ac{DD}, however we take into account the effect of \ac{GD} \emph{optimization}, which was missed by virtually all the previous literature due to their focusing on analytic solutions.
The effect of the smallest non-zero eigenvalue on \ac{DD}, through a condition number, was briefly noticed by~\cite{rangamani2020interpolating}. In this work we carry out a more comprehensive analysis and show how the excess risk of \ac{GD} is controlled the smallest eigenvalue.
In particular, $\lminh^+$ has a ``U''-shaped behaviour as the number of features increases, and we give a high-probability characterization of this behavior when inputs are subgaussian.
To some extent, this is a non-asymptotic manifestation of the Bai-Yin law, whose connection to \ac{DD} in an asymptotic setting was noted by \citet[Theorem 1]{hastie2019surprises}.

Some interest was also dedicated to the effect of bias and variance of \ac{DD} \citep{mei2019generalization} in the same pseudo-inverse setting, while more involved fine-grained analysis was later carried out by \cite{adlam2020understanding}.
In this work we focus on the influence of the optimization error, which is complementary to the bias-variance effects (typically we care about it once optimization error is negligible).

\ac{DD} behaviour was also observed beyond least squares, in neural networks and other interpolating models~\citep{belkin2019reconciling}.
To some extent a formal connection to neural networks was first made by \citet{mei2019generalization} who studied asymptotic behaviour of the risk under the random feature model, when $n,d^{\mathrm{input}},d^{\mathrm{RF}} \to \infty$ while having $\frac{n}{d^{\mathrm{input}}}$ and $\frac{d^{\mathrm{RF}}}{d^{\mathrm{input}}}$ fixed.
Later on, with popularity of \ac{NTK} the connection became clearer as within \ac{NTK} interpretation shallow neural networks can be paralleled with kernelized predictors~\citep{bartlett2021deep}.
A detailed experimental study of \ac{DD} in deep neural networks was carried out by~\citep{nakkiran2019deep}, who showed that various forms of regularization mitigate \ac{DD}.
In this work, we explain \ac{DD} in least-squares solution obtained by \ac{GD} through the spectrum of the features, where optimization error has a visible role.
While we do not present formal results for neural networks, but we empirically investigate whether our conclusions extend to shallow neural nets as would be suggested by \ac{NTK} theory.

\section{Excess Risk of the Gradient Descent Solution}
\label{sec:excess}
We focus on learners that optimize parameters via the \acl{GD} algorithm.
We treat \ac{GD} as a measurable map $\algmap : \sS \times \reals^d \to \reals^d$, where $\sS = \sZ^n$ is the space of size-$n$ training sets. Given a training set $S \in \sS$ and an initialization point $\bw_0 \in \sW$, we write $\alg(\bw_0)$ to indicate the output obtained recursively by the standard
\ac{GD} update rule with some fixed step size $\alpha > 0$, i.e.\ $\alg(\bw_0) = \bw_T$, where
\[
  \bw_t = \bw_{t-1} - \alpha \nabla \wh{L}_S(\bw_{t-1}), \qquad t=1,\ldots,T~.
\]

We look at the behavior of \ac{GD} in the \emph{overparameterized} regime ($d > n$) when the
initialization parameters are sampled from an isotropic Gaussian density, that is $\bW_0 \sim \sN(\bzero, \initvar\bI_{d\times d})$ with some initialization variance $\initvar$.
It is well-known that in the overparameterized regime, \ac{GD} is able to achieve zero empirical loss.  Therefore, rather than focusing on the generalization gap $L(\alg(\bW_0)) - \Lh_S(\alg(\bW_0))$ it is natural to compare the loss of $\alg(\bW_0)$ to that of the \emph{best} possible predictor.  Thus, we consider the \emph{excess risk}
defined as
\[
  \sE(\bwstar) = L(\alg(\bW_0)) - L(\bwstar)~, \qquad \bwstar \in \argmin_{\bw \in \reals^d} L(\bw)~.
\]
Our results are based on a the requirement that $\alg$ satisfies the following regularity condition:
\begin{definition}
  \label{def:admissible}
  A map $f : \reals^d \to \reals^d$ is called $(\Delta, \bM)$-admissible, where $\bM$ is a fixed \ac{PSD} matrix and $\Delta \geq 0$, 
  if for all $\bw, \bw' \in \reals^d$ the following holds:
  \[
    \|f(\bw) - f(\bw')\|_{\bM} \leq \Lip \|\bw - \bw'\|~.
  \]
\end{definition}
Notice that the norm on the left-hand side is $\|\cdot\|_{\bM}$, while that on the right-hand side is the standard Euclidean norm.
Also note that this inequality entails a Lipschitz condition with Lipschitz factor $\Delta$.

Our first main result gives an upper bound on the excess risk of \ac{GD} output, assuming that the output of $\alg$ is of \emph{low-rank}, in the sense that for some low-rank orthogonal projection $\bM \in \reals^{d \times d}$
we assume that $\bM \alg(\bw) = \alg(\bw)$ almost surely (a.s.) with respect to $S$, for any initialization $\bw$.
This condition is of interest in the overparameterized regime, where the learning dynamics effectively happens in a subspace which is arguably of much smaller dimension than the whole parameter space.
The following theorem bounds the excess risk (with respect to a possibly non-convex but smooth loss) of any algorithm that satisfies \cref{def:admissible} with some $(\Delta, \bM)$.
Later it will become apparent that in a particular learning problem this pair consists of data-dependent quantities.
Importantly, the theorem demonstrates how the excess risk is controlled by 
the learning dynamics on the subspace spanned by $\bM$ (the first and the second terms on the right hand side).
It also shows how much is lost due to not learning on the complementary subspace (the third term).
The first two terms will become crucial in our analysis of the double descent, while we will show that the last term will vanish as $n \to \infty$. 
\begin{theorem}[Excess Risk]
  \label{thm:excess_risk}
  Assume that $\bW_0 \sim \sN(\bzero, \initvar\bI_{d\times d})$,
  and assume that $\alg$ is $(\Delta, \bM)$-admissible (\cref{def:admissible}), where $\Delta$ and $\bW_0$ are independent.
  Further assume $\bM \alg(\bw) = \alg(\bw)$ for any $\bw$, 
  and that $L$ and $\Lh$ are $H$-smooth.
  Then, for any $\bwstar \in \argmin_{\bw \in \reals^d} L(\bw)$ we have
  \[
    \E[\sE(\bwstar)]
    \leq
    H \pr{
      \underbrace{
      \E[\Delta^2] 
      \pr{ \|\bwstar\|^2 + \initvar (2 + d) }
      }_{(1)}
      +
      \underbrace{
      \E[\|\alg(\bwstar) - \bwstar\|_{\bM}^2]
      }_{(2)}
      +
      \frac12 
      \underbrace{
      \E[\|\bwstar\|^2_{\bI-\bM}]
      }_{(3)}
    }~.
  \]
  In particular for \ac{GD}, having $\alpha \leq 1/H$,
  \begin{align*}
    \E[\|\alg(\bwstar) - \bwstar\|_{\bM}^2]
    \leq 2 \alpha T L(\bwstar)~.
  \end{align*}
\end{theorem}
The proof is in \cref{app:excess_risk_gd}.
The main steps are using the $H$-smoothness of $L$ to upper-bound $\sE(\bwstar)$ in terms of the squared norm of $\alg(\bW_0) - \bwstar$ and decomposing the latter as the sum of the squared norms of its projections onto the space spanned by $M$ and its orthogonal complement, by the Pythagorean theorem.
Then $\alg(\bW_0) - \bwstar = \alg(\bW_0) - \alg(\bwstar) + \alg(\bwstar) - \bwstar$ is used on the subspace spanned by $\bM$:
the norm of $\alg(\bW_0) - \alg(\bwstar)$ is controlled by using the admissibility of $\alg$ and Gaussian integration, and the norm of $\alg(\bwstar) - \bwstar$ is controlled by the accumulated squared norms of gradients of $\Lh_S$ over $T$ steps of gradient descent, which is conveniently bounded by $2\alpha T \Lh_S(\bwstar)$ when $\alpha \leq 1/H$ due to the $H$-smoothness of $\Lh_S$.

We will rely on \cref{thm:excess_risk} for our analysis of the Least-Squares problem as follows.

\subsection{Least-Squares with Random Design and No Label Noise}
\label{sec:ls_random_design}
Consider a noise-free linear regression model with random design:
\[
  Y = \bX\tp \bwstar
\]
where instances $\bX$ are distributed according to some unknown distribution $P_X$ supported on a $d$-dimensional unit Euclidean ball.
After observing a training sample $S = \pr{(\bX_i, Y_i)}_{i=1}^n$, we run \ac{GD} on the given empirical square loss
\[
  \Lh_S(\bw) = \frac{1}{2 n} \sum_{i=1}^n (\bw\tp \bX_i - Y_i)^2~.
\]
In the setting of our interest, the sample covariance matrix $\bhSigma = (\bX_1 \bX_1\tp + \dots + \bX_n \bX_n\tp) / n$ might be degenerate, and therefore we will occasionally refer to the non-degenerate subspace $\bU_r = [\bu_1, \ldots, \bu_r]$, where $\bU$ is given by the \ac{SVD}: $\bhSigma = \bU \bS \bV\tp$ and $\bu_1, \ldots, \bu_r$ are the eigenvectors corresponding to the eigenvalues 
$\hat{\lambda}_1,\ldots,\hat{\lambda}_r$, where $\hat{\lambda}_i = \lambda_i(\bhSigma)$, arranged in decreasing order:
\[
\lambda_1(\bhSigma) \geq 
\lambda_2(\bhSigma) \geq \cdots \geq
\lambda_r(\bhSigma) > 0
\]
and $r = \rank(\bhSigma)$.
We write $\lminh^+ = \lmin^+(\bhSigma) = \lambda_r(\bhSigma)$ for the minimal \emph{non-zero} eigenvalue, and
we denote $\bhatM = \bU_r \bU_r\tp$.
Note that $\bhatM^2 = \bhatM$.
Now we state our main result in this setting.

\begin{theorem}
  \label{thm:excess_ls_noiseless}
  Assume that $\bW_0 \sim \sN(\bzero, \initvar \bI)$.
  Then, for any $\bwstar \in \argmin_{\bw \in \reals^d} L(\bw)$ and any $x>0$, with probability $1-e^{-x}$ over random samples $S$ we have
  \[
    \E[\sE(\bwstar)]
    \leq
      \E\br{(1 - \alpha \lminh^+)^{2T}} \pr{ \|\bwstar\|^2 + \initvar (2 + d) }
      +
    \frac12 
      \E[\|\bwstar\|^2_{\bI-\bhatM}]~.
  \]
\end{theorem}

The proof is in \cref{app:excess_risk_gd}. 
This is a consequence of \cref{thm:excess_risk}, modulo showing that \ac{GD} with the least squares objective is $(\Delta, \bhatM)$-admissible with $\Delta = (1-\alpha \lminh^+)^T$, and upper-bounding $\E[\|\bwstar\|^2_{\bI-\bhatM}]$ by controlling  the expected squared norm of the projection onto the orthogonal complement of the space spanned by $\bU_r$.
The later comes up in the analysis of PCA (see e.g.
\citet[Theorem 1]{shawe2005eigenspectrum}) and, as we show in \cref{app:sec:third_term}, this term is expected to be small enough whenever the eigenvalues have exponential decay, in which case with high probability we have $\E[\|\bwstar\|^2_{\bI-\bhatM}] \lesssim \|\bwstar\|_2^2 / \sqrt{n}$ as $n \to \infty$. 
Note that the middle term in the upper bound of our \cref{thm:excess_risk} vanishes in the noise-free case: $\E[\|\alg(\bwstar) - \bwstar\|_{\bhatM}^2] = 0$.

Looking at \cref{thm:excess_ls_noiseless}, we can see that the excess risk is bounded by the sum of two terms. 
Note that the second term is negligible in many cases (consider the limit of infinite data) and additionally it is a term that remains constant during training as it does not depend on training data. 
Therefore, we are particularly interested in the first term of the bound, which is data-dependent. This term depends on  $\lminh^+$ via a functional form that has a double descent behaviour if plotted against $d$ for fixed $n$.
Before going into that analysis, let us also consider the scenario with label noise.

\subsection{Least-Squares with Random Design and Label Noise}
\label{sec:ls_random_design_plus_noise}
Now, in addition to the random design we introduce  label noise into our model:
\[
  Y = \bX\tp \bwstar + \ve~,
\]
where we have random noise $\ve$ such that $\E[\ve] = 0$ and $\E[\ve^2] = \noisevar$, independent of the instances.

\begin{theorem}
  \label{thm:excess_ls_noise}
  Assume that $\bW_0 \sim \sN(\bzero, \initvar \bI)$.
  Then, for any $\bwstar \in \argmin_{\bw \in \reals^d} L(\bw)$ and any $x>0$, with probability $1-e^{-x}$ over random samples $S$ we have
  \[
    \E[\sE(\bwstar)]
    \leq
      \E\br{(1 - \alpha \lminh^+)^{2T}} \pr{ \|\bwstar\|^2 + \initvar (2 + d) }
      +
      \frac{4 \sigma^2}{n} \E\br{\pr{\lminh^+}^{-2}}
      +
    \frac12 
      \E[\|\bwstar\|^2_{\bI-\bhatM}]~.
  \]
\end{theorem}
The proof is in \cref{app:excess_risk_gd}. 
Again, this follows from \cref{thm:excess_risk}, by the same steps used in the proof of 
\cref{thm:excess_ls_noiseless},
except that the term $\E\br{ \|\bwstar - \sA_S(\bwstar)\|_{\bhatM}^2 }$ is now handled by conditioning on the sample and analyzing the expectation with respect to the random noise (\cref{lem:ls_noise} and its proof in \cref{app:sec:ls_random_design_withnoise}), leading to the new term $\frac{4 \sigma^2}{n} \E\br{\bigl(\lminh^+\bigr)^{-2}}$. The latter closely resembles the term one would get for ridge regression \citep[Cor.\ 13.7]{shalev2014understanding} due to algorithmic stability~\citep{bousquet2002stability}, but here we have a dependence on the smallest non-zero eigenvalue instead of a regularization parameter.

\subsection{Concentration of the Smallest Non-zero Eigenvalue}
\label{sec:lmin_concentration}
In this section we take a look at the behaviour of $\lminh^+$ assuming that input instances $\bX_1, \ldots, \bX_n$ are i.i.d.\ \emph{random} vectors, sampled from some underlying marginal density that meets some regularity requirements
(Definitions~\ref{def:subgaussian} and~\ref{def:isotropic} below)
so that we may use the results from random matrix theory~\citep{vershynin2010}. 
Recall that the covariance matrix of the input features is $\bhSigma = (\bX_1 \bX_1\tp + \dots + \bX_n \bX_n\tp) / n$.
We focus on the concentration of $\lminh^+ = \lmin^+(\bhSigma)$ around its population counterpart $\lmin^+ = \lmin^+(\bSigma)$, where $\bSigma$ is the population covariance matrix: $\bSigma = \E[\bX_1 \bX_1\tp]$.

In particular, the Bai-Yin limit characterization of the extreme
eigenvalues of sample covariance matrices~\citep{bai1993limit} implies that 
$\lminh^+$ has almost surely an asymptotic behavior $(1-\sqrt{d/n})^2$ as the dimensions grow to infinity, assuming that the matrix $\bX := [\bX_1, \ldots, \bX_n] \in \reals^{d \times n}$ has independent entries.
We are interested in the non-asymptotic version of this result.
However, unlike \cite{bai1993limit}, we do not assume independence of all entries, but rather independence of observation vectors (columns of $\bX$). This will be done by introducing a distributional assumption: we assume that observations are \emph{sub-Gaussian} and \emph{isotropic} random vectors.
\begin{definition}[Sub-Gaussian random vectors]
\label{def:subgaussian}
  A random vector $\bX \in \reals^d$ is sub-Gaussian if the random variables $\bX\tp \by$ are sub-Gaussian for all $\by \in \reals^d$.
  The sub-Gaussian norm of a random vector $\bX \in \reals^d$ is defined as
  \[
    \|\bX\|_{\psi_2} = \sup_{\|\by\| = 1}\sup_{p \geq 1}\cbr{ \frac{1}{\sqrt{p}} \E[|\bX\tp \by|^p]^{\frac1p} }~.
  \]
\end{definition}
\begin{definition}[Isotropic random vectors]
\label{def:isotropic}
  A random vector $\bX \in \reals^d$ is called isotropic if its covariance is the identity: $\E\br{\bX \bX\tp} = \bI$.
  Equivalently, $\bX$ is isotropic if $\E[(\bX\tp \bx)^2] = \|\bx\|^2$ for all $\bx \in \reals^d$.
\end{definition}
Let $\bSigma^{\pinv}$ be the Moore-Penrose pseudoinverse of $\bSigma$.
In \cref{sec:lmin_concentration_proof} we prove the following.\footnote{$(x)_+ = \max\cbr{0, x}$}
\begin{lemma}[Smallest non-zero eigenvalue of sample covariance matrix]
  \label{lem:non_asymptotic_bai_yin}
  Let $\bX = [\bX_1, \ldots, \bX_n] \in \reals^{d \times n}$ be a matrix with i.i.d.\ columns, such that $\max_i\|\bX_i\|_{\psi_2} \leq K$,
  and let $\bhSigma = \bX \bX\tp / n$, and $\bSigma = \E[\bX_1 \bX_1\tp]$.
  Then, for every $x \geq 0$, with probability at least $1-2e^{-x}$, we have
  \[
    \lmin^+(\bhSigma)
    \geq
    \lmin^+(\bSigma) \pr{1 - K^2 \pr{c \sqrt{\frac{d}{n}} + \sqrt{\frac{x}{n}}}}_+^2
    \qquad
    \text{for } n \geq d~,
  \]
  and furthermore, assuming that $\|\bX_i\|_{\bSigma^{\pinv}} = \sqrt{d}$ \ a.s. for all $i \in [n]$, we have
  \[
    \lmin^+(\bhSigma)
    \geq
    \lmin^+(\bSigma) \pr{\sqrt{\frac{d}{n}} - K^2 \pr{c + 6 \sqrt{\frac{x}{n}}}}_+^2
    \qquad
    \text{for } n < d~,
  \]
  where we have an absolute constant $c = 2^{3.5} \sqrt{\ln(9)}$.
\end{lemma}

\cref{lem:non_asymptotic_bai_yin} is a non-asymptotic result that allows us to understand the behaviour of $\lminh^+$, and hence the behaviour of the excess risk that depends on this quantity, for fixed dimensions. We will exploit this fact in the following section in which we discuss the implications of our findings.

\section{Excess risk as a function of over-parameterization}
\label{sec:excess_risk}

First we note that, in the noise-free case, the middle term in the upper bound of \cref{thm:excess_risk} vanishes: $\E[\|\alg(\bwstar) - \bwstar\|_{\bhatM}^2] = 0$.
Thus, as in \cref{thm:excess_ls_noiseless}, the upper bound consists only of the term involving the smallest positive eigenvalue $\lminh^+$ and the term involving $\E[\|\bwstar\|^2_{\bI-\bhatM}]$. 
The behaviour of the former was clarified in ~\cref{sec:lmin_concentration},
and the latter is controlled as explained in \cref{app:sec:third_term}.
Thus,
in the \emph{overparametrized} regime ($d>n$) we have:
\footnote{We use $f \leqC g$ when there exists a universal constant $C > 0$ such
that $f \leq C g$ uniformly over all arguments.} 
\begin{equation*}
  \E[\sE(\bwstar)]
  \leqC
    \pr{1- \frac{\alpha}{n} (\sqrt{d} - \sqrt{n} - 1)_+^2}^{2T} \|\bwstar\|^2
    +
    \E\br{\|\bwstar\|^2_{\bI-\bhatM}}~.
\end{equation*}
A similar bound holds in the \emph{underparameterized} case ($d < n$) but replacing the term $(\sqrt{d} - \sqrt{n} - 1)_+^2$ with $(\sqrt{n} - \sqrt{d} - 1)_+^2$.
Note that the term multiplying the learning rate is $(\sqrt{d/n} - 1 - 1/\sqrt{n})^2$, in accordance with the Bai-Yin limit which says that asymptotically $\lminh^+ \sim (\sqrt{d/n} - 1)^2$.
It is interesting to see how $\bigl(1-\alpha(\sqrt{d/n} - 1)_+^2\bigr)^{2T}$ varies with model size $d$ for a given fixed dataset size $n$ and fixed number of gradient updates $T$. Setting $y = d/n$ and considering the cases $y \to 0$ (underparameterized regime), $y \sim 1$ (the peak), and $y > 1$ (overparameterized regime) it becomes evident that this term has a double descent behaviour.
Thus, the double descent is captured in the part of the excess risk bound that corresponds to learning dynamics on the space spanned by $\bhatM$.

Similarly, we can now consider the scenario with label noise:
we can similarly bound the excess risk, following the same logic as for noise-free case; however we have an additional dependence on $\sigma^2$ via the term $\frac{4 \sigma^2}{n} \E\Bigl[\bigl(\lminh^+\bigr)^{-2}\Bigr]$. While this does not interfere with the \ac{DD} shape as we change model size, it does imply that the peak is dependent on the amount of noise. In particular, the more noise we have in the learning problem the larger we expect the peak at the interpolation boundary to be.

While the presence of the double descent has been studied by several works, our derivation provides two potentially new interesting insights. The first one is that there is a dependency between the noise in the learning problem and the shape of the curve, the larger the noise is, the larger the peak in DD curve. This agrees with the typical intuition in the underparmetrized regime that the model fits the noise when it has enough capacity, leading towards a spike in test error. However, due to the dependence on $\lminh^+$, it is subdued as the model size grows.
Secondly, and maybe considerably more interesting, there seems to be a connection between the double descent curve of the excess risk and the optimization process. In particular, our derivation is specific to gradient descent. In this case the excess risk seems to depend on the conditioning of the features in the least squares problem on the subspace spanned by the data through $\lminh^+$, which also affects convergence of the optimization process. For the least squares problem this can easily be seen, as the sample covariance of the features corresponds to the Gauss-Newton approximation of the Hessian~\cite[e.g.][]{NoceWrig06}, hence it impacts the convergence. In a more precise way, conditioning of any matrix is measured by the ratio $s_{\max}/s_{\min}$ (the `condition number') which is determined solely by the smallest singular value $s_{\min}$ in cases when $s_{\max}$ is of constant order, such as the case that we studied here: Note that by our boundedness assumption, $s_{\max}$ is constant,
but in general one needs to consider both $s_{\max}$ and $s_{\min}$ in order to characterize the condition numbers, which interestingly have been observed to display a double descent as well \cite{poggio2019double}.

More generally, normalization, standardization, whitening and various other preprocessing of the input data have been a default step in many computer vision systems~\cite[e.g.][]{lecun98b, Krizhevsky09learningmultiple} where it has been shown empirically that they greatly affect learning. Such preprocessing techniques are usually aimed to improve conditioning of the data. Furthermore, various normalization layers like batch-norm~\citep{Ioffe15} or layer-norm~\citep{Ba16} are typical components of recent architectures, ensuring that features of intermediary layers are well conditioned.
Furthermore, it has been suggested that model size improves conditioning of the learning problem \citep{Li18}, which is in line with our expectation given the behaviour of $\lminh^+$. Taking inspiration from the optimization literature, it is natural for us to ask whether for neural networks, we can also connect the conditioning or $\lminh^+$ of intermediary features and double descent. This particular might be significant if we think of the last layer of the architecture as a least squares problem (assuming we are working with mean square error), and all previous layers as some random projection, ignoring that learning is affecting this projection as well.

This relationship between generalization and double descent on one hand, and the conditioning of the features and optimization process raises some additional interesting questions, particularly since, compared to the typical least squares setting, the conditioning of the problem for deep architectures does not solely depend on size. In the next section we empirically look at some of these questions.

\section{Empirical exploration in neural networks}
\label{sec:empirical_discussion}

The first natural question to ask is whether the observed behaviour for the least squares problem is reflected when working with neural networks. To explore this hypothesis, and to allow tractability of computing various quantities of interest (like $\lminh^+$), we focus on one hidden layer MLPs on the MNIST and FashionMNIST datasets. We follow the protocol used by \cite{belkin2019reconciling}, relying on a squared error loss.
In order to increase the model size we simply increase the dimensionality of the latent space, and rely on gradient descent with a fixed learning rate and a training set to $1000$ randomly chosen examples for both datasets. More details can be found in \cref{sec:app:empirical}.

\begin{figure}
\captionsetup[subfloat]{labelformat=empty}
\centering
\subfloat[MNIST]{
    \begin{minipage}[b]{.2\linewidth}
      \centering
      \includegraphics[align=c,width=\textwidth]{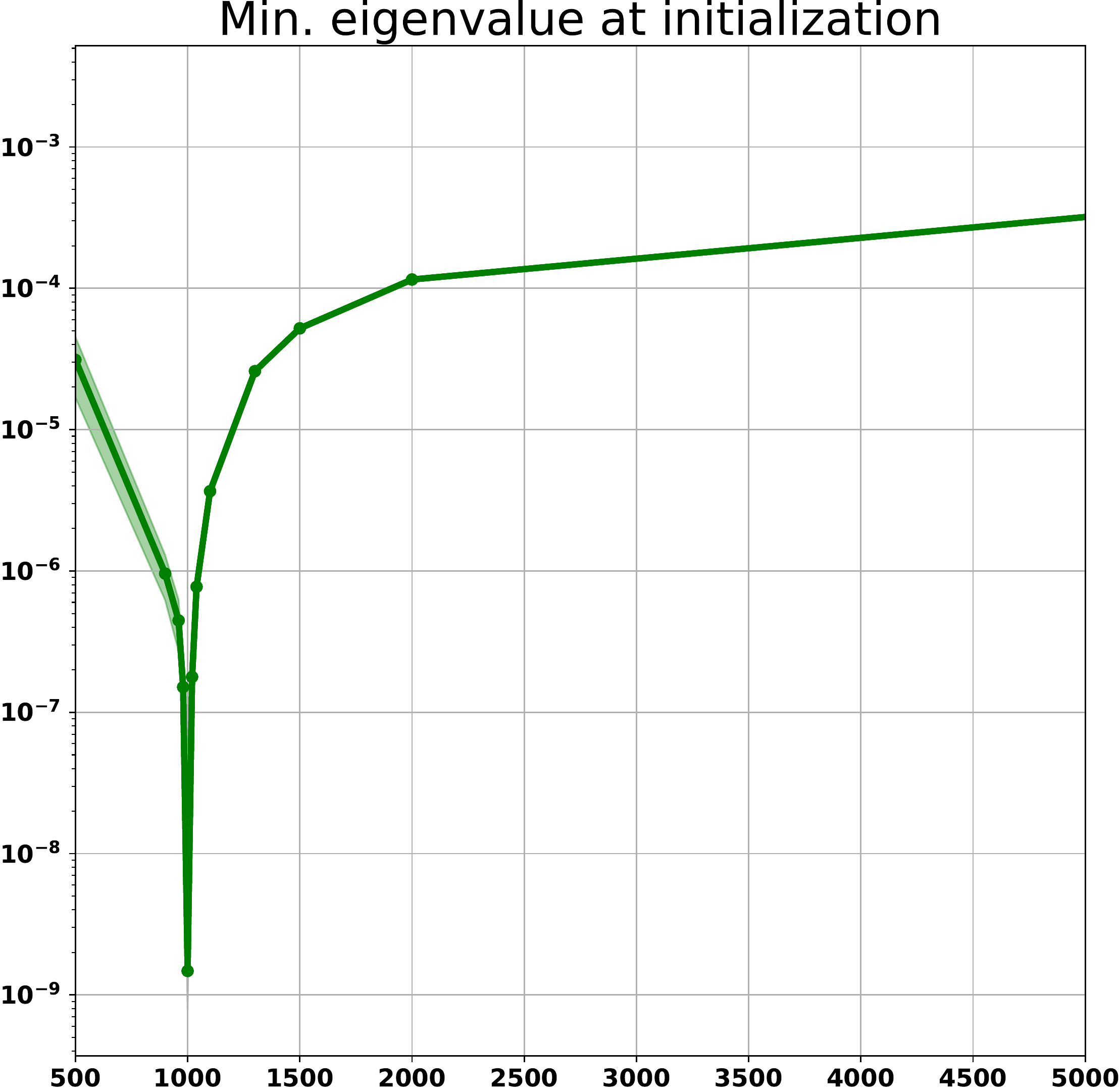}
      (a)
        \label{fig:MNIST_1l_eig_init}
        \end{minipage} \hfill
    \begin{minipage}[b]{.67\linewidth}
        \centering
        \includegraphics[align=c,width=\textwidth]{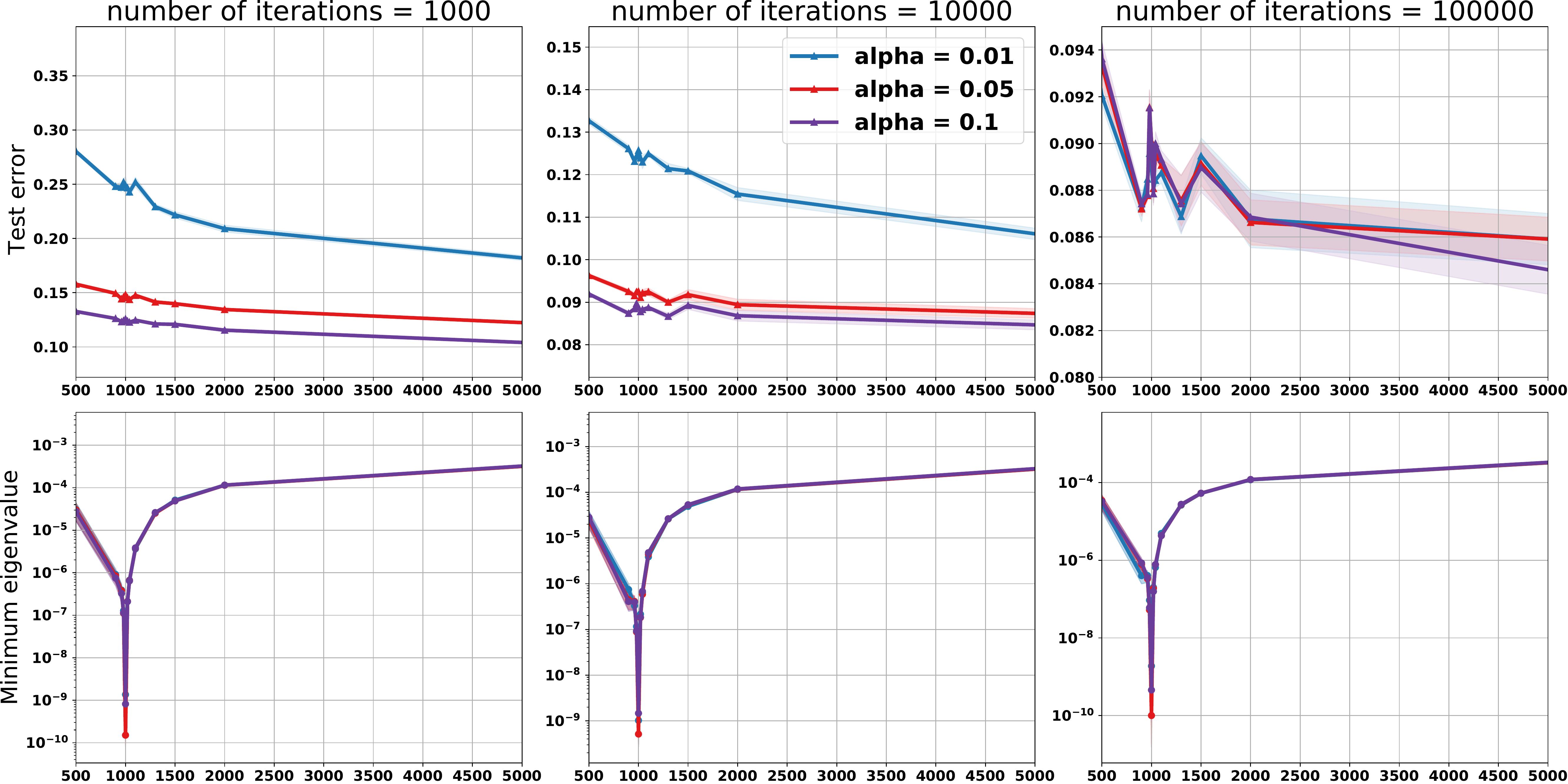}
        (b)
        \label{fig:MNIST_1l_error}
        \end{minipage}
    }\\
\subfloat[FashionMNIST]{
    \begin{minipage}[b]{.2\linewidth}
        \centering
        \includegraphics[align=c,width=\textwidth]{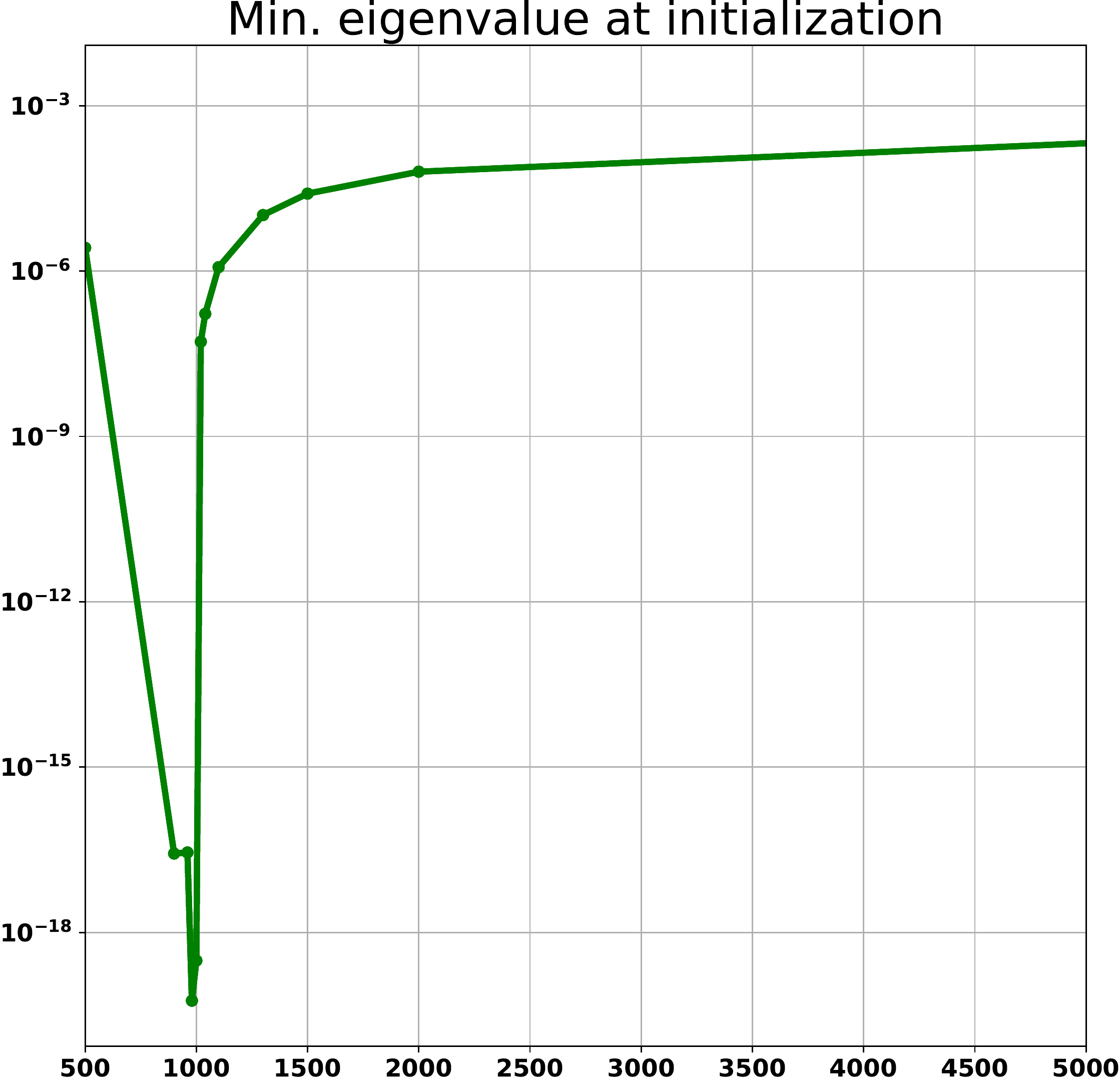}
        (a)
        \end{minipage} \hfill
    \begin{minipage}[b]{.67\linewidth}
        \centering
        \includegraphics[align=c,width=\textwidth]{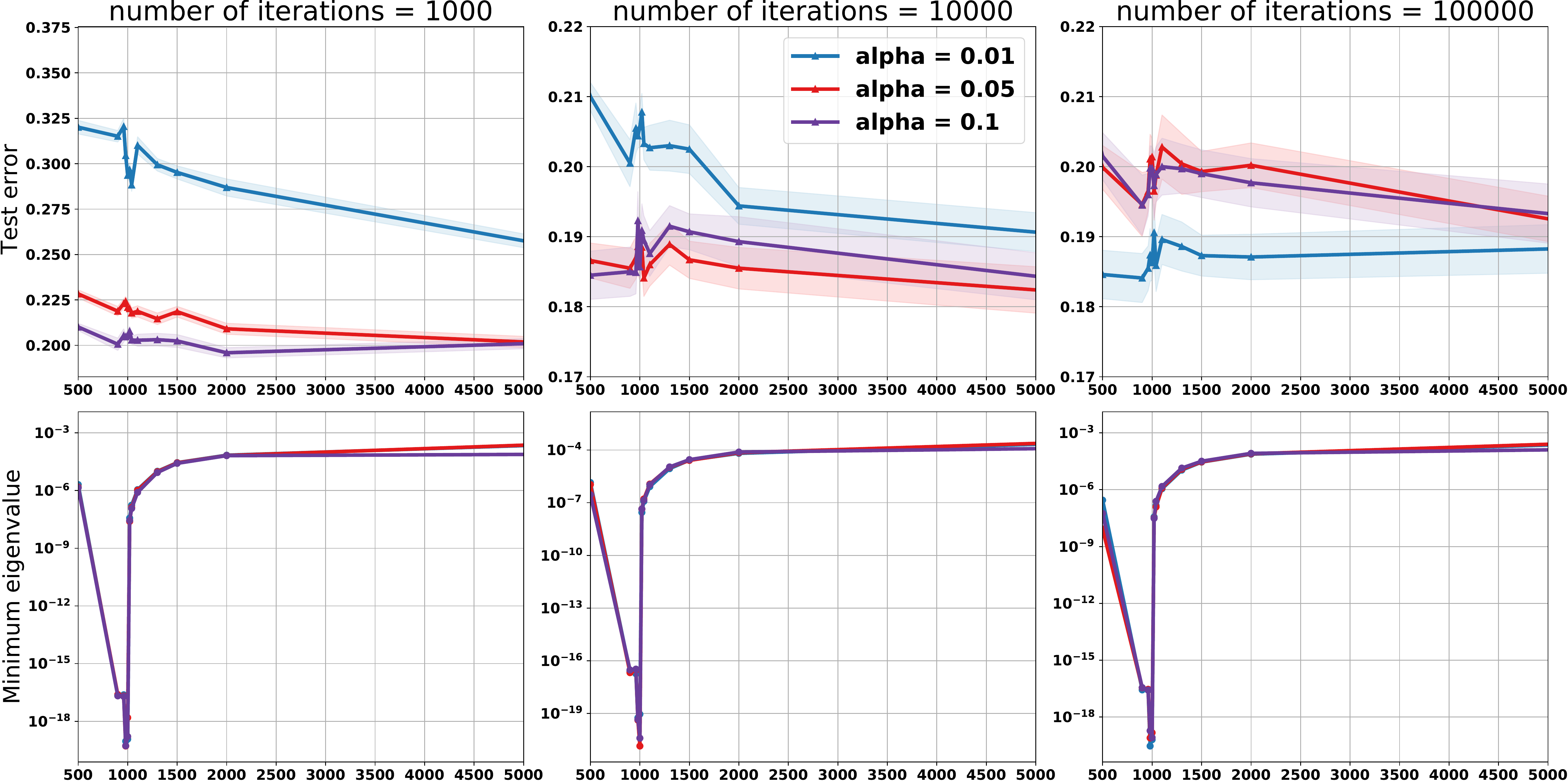}
        (b)
        \label{fig:FashionMNIST_1l_error}
        \end{minipage}
}
\caption{Training one hidden layer networks of increasing width on MNIST (top) and FashionMNIST (bottom): (a) Minimum positive eigenvalue of the intermediary features at initialization - (b) Test error and corresponding minimum eigenvalue of the intermediary features at different iterations}\label{fig:MNIST_FashionMNIST}
\end{figure}

Figure~\ref{fig:MNIST_FashionMNIST} provides the main findings on this experiment. Similar to the Figure~\ref{fig:ls_emp_bound_vs_sim}, we depict $3$ columns showing snapshots at different number of gradient updates: $1000$, $10000$ and $100000$. The first row shows 
test error (number of miss-classified examples out of the test examples) computed on the full test set of $10000$ data points which as expected shows the double descent curve with a peak around $1000$ hidden units. Note that the peak is relatively small, however the behaviour seems consistent under $5$ random seeds for the MNIST experiment.\footnote{The error bars for the test error in all the other experiments are estimated by splitting the test set into 10 subsets.} The second row and potentially the more interesting one looks at the $\lminh^+$ computed on the covariance of the activations of the hidden layer, which as predicted by our theoretical derivation shows a dip around the interpolation threshold, giving the expected U-shape. Even more surprisingly this shape seems to be robust throughout learning, and the fact that the input weights and biases are being trained seems not to alter it, thus suggesting that our derivation might provide insights in the behaviour of deep models. 

\begin{figure}[ht]
    \begin{minipage}[b]{.2\linewidth}
        \centering
        \includegraphics[align=c,width=\textwidth]{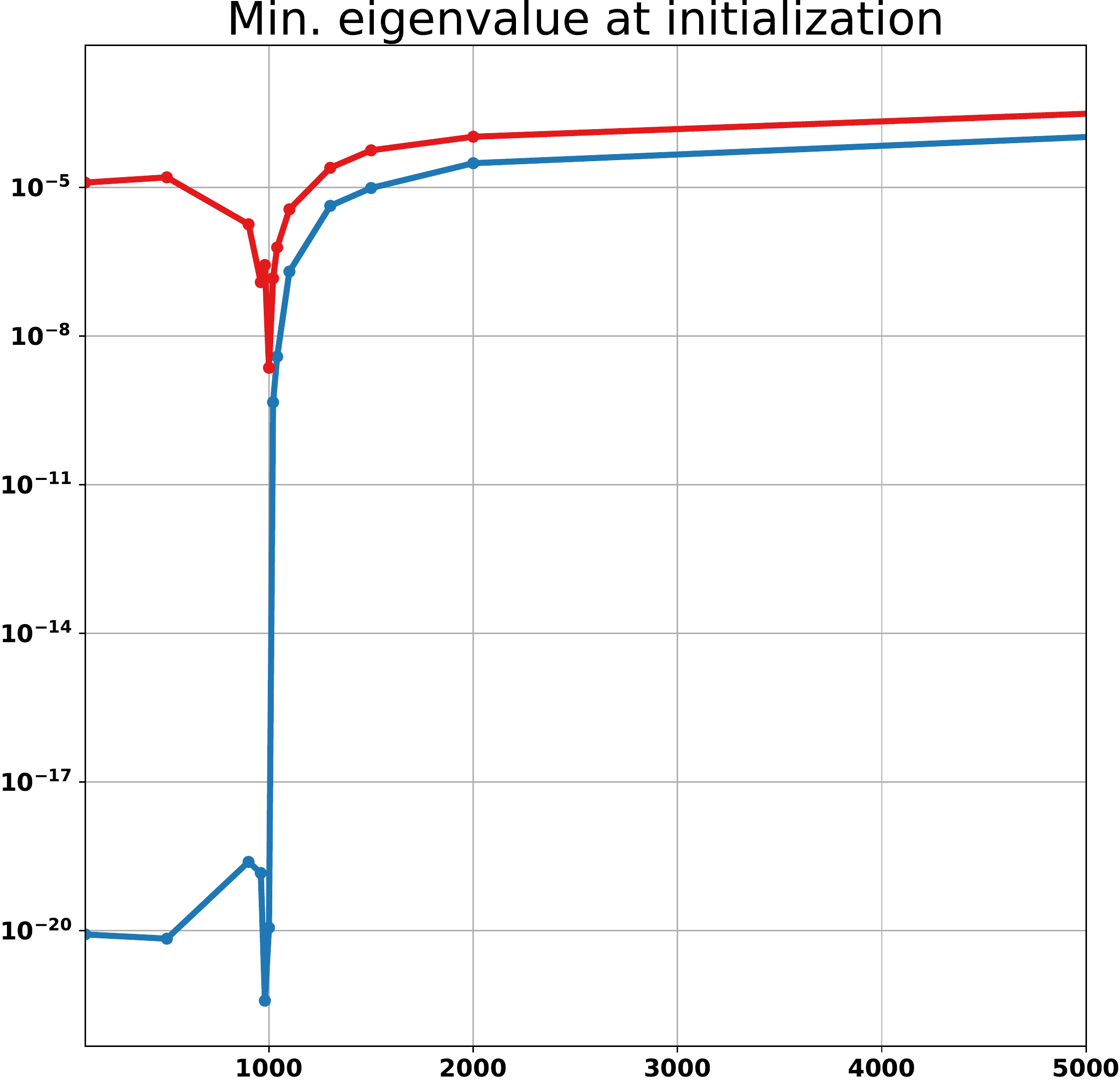}
        (a)
        \label{fig:MNIST_depth_eig_init}
        \end{minipage} \hfill
    \begin{minipage}[b]{.7\linewidth}
        \centering
        \includegraphics[align=c,width=\textwidth]{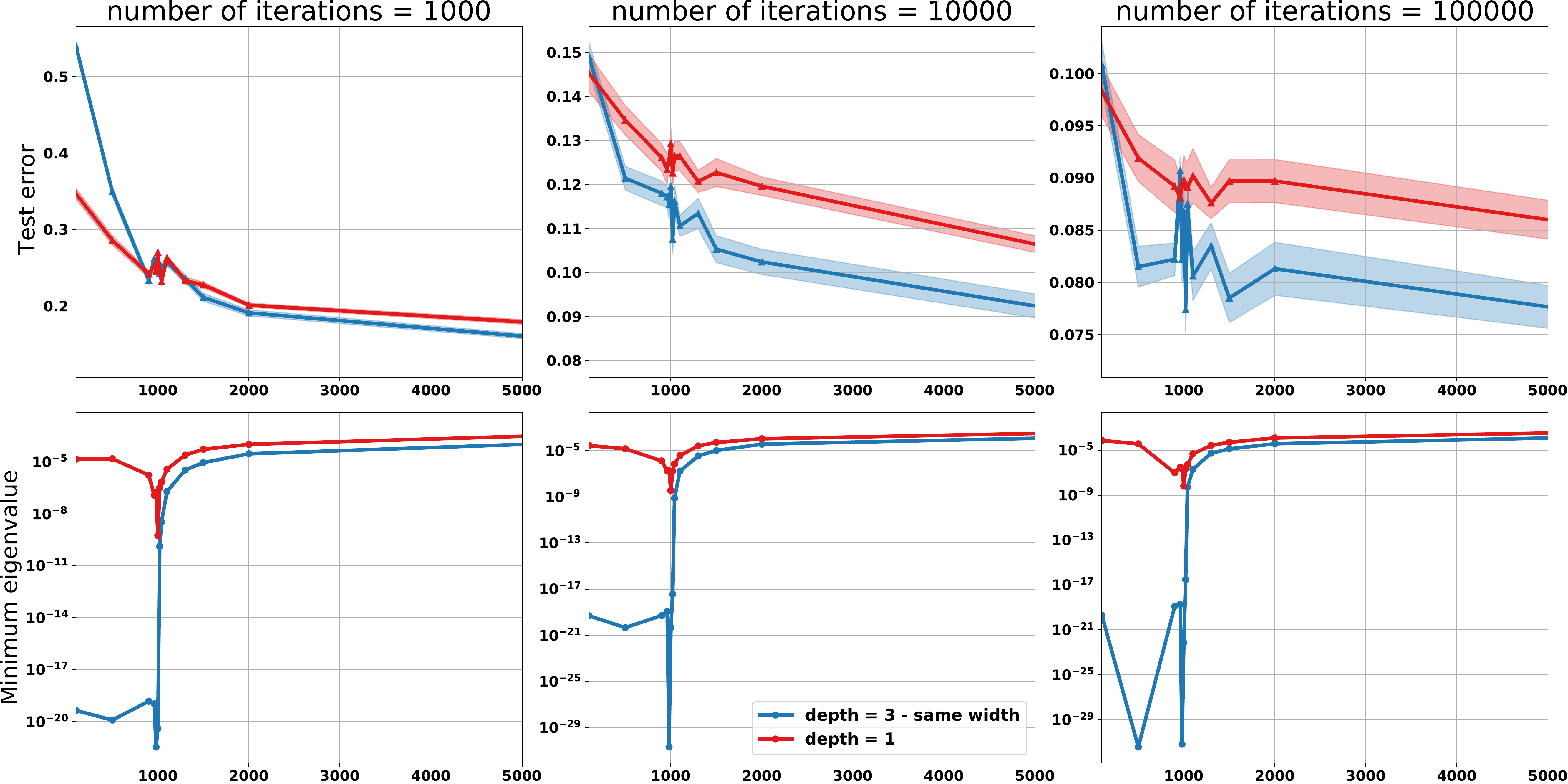}
        (b)
        \label{fig:MNIST_depth_error}
        \end{minipage}
    \caption{Training networks of increasing width with 1 and 3 hidden layers on MNIST: (a) Minimum positive eigenvalue of the intermediary features at initialization - (b) Test error and corresponding minimum eigenvalue of the intermediary features at different iterations}\label{fig:MNIST_depth}
\end{figure}

Following this, if we think of the output layer as solving a least squares problem, while the rest of the network provides a projection of the data, we can consider what can affect the conditioning of the last latent space of the network. We put forward the hypothesis that $\lminh^+$ is not simply affected by the number of parameters, but actually the distribution of these parameters in the architecture matters. 

To test this hypothesis, we conduct an experiment where we compare the behavior of a network with a single hidden layer and a network with three hidden layers. For both networks, we increase the size of the hidden layers. For the deeper network, we consider either increasing the size of all the hidden layers or grow only the last hidden layer while keeping the others to a fixed small size, creating a strong bottleneck in the network. Figure~\ref{fig:MNIST_depth} shows the results obtained with the former, while the effect of the bottleneck can be seen
in \cref{app:depth_exp}. We first observe that for the three tested networks, the drop in the minimum eigenvalues happens when the size of the last hidden layer reaches the number of training samples, as predicted by the theory. The magnitude of this drop and behavior across the different tested sizes depends however on the previous layers. In particular, we observe that the bottleneck yields features that are more ill-conditioned than the network with wide hidden layers, where the width of the last layer on its own can not compensate for the existence of the bottleneck.
Moreover, from Figure~\ref{fig:MNIST_depth}, we can clearly see that the features obtained by the deeper network have a bigger drop in the minimum eigenvalue, which results, as expected in a higher increase in the test error around the interpolation threshold.

It is well known that depth can harm optimization making the problem ill-conditioned, hence the reliance on skip-connections and batch normalization~\cite{De2020} to train very deep architecture. Our construction provides a way of reasoning about double descent that allows us to factor in the ill-conditioning of the learning problem. Rather than focusing simply on the model size, it suggests that for neural networks the quantity of interest might also be $\lminh^+$ for intermediary features, which is affected by size of the model but also by the distribution of the weights and architectural choices. For now we present more empirical explorations and ablations in \cref{sec:app:empirical}, and put forward this perspective as a conjecture for further exploration.

\section{Conclusion and Future Work}
\label{sec:conclusions}

In this work we analyse the double descent phenomenon in the context of the least squares problem. We make the observation that the excess risk of gradient descent is controlled by the smallest \emph{positive} eigenvalue, $\lminh^+$, of the feature covariance matrix. Furthermore, this quantity follows the Bai-Yin law with high probability under mild distributional assumptions on features, that is, it manifests a U-shaped behaviour as the number of features increases, which we argue induces a double descent shape of the excess risk. Through this we provide a connection between the widely known phenomena and optimization process and conditioning of the problem. We believe this insight provides a different perspective compared to existing results focusing on the Moore-Penrose pseudo-inverse solution.
In particular our work conjectures that the connection between the known double descent shape and model size is through $\lminh^+$ of the features at intermediary layers. For the least squares problem $\lminh^+$ correlates strongly with model size (and hence feature size). However this might not necessarily be always true for neural networks. For example we show empirically that while both depth and width increase the model size, they might affect $\lminh^+$ differently. We believe that our work could enable much needed effort, either empirical or theoretical, to disentangle further the role of various factors, like depth and width or other architectural choices like skip connections on double descent.  

\bibliographystyle{unsrtnat}
\bibliography{learn_min}

\begin{thebibliography}{30}
\providecommand{\natexlab}[1]{#1}
\providecommand{\url}[1]{\texttt{#1}}
\expandafter\ifx\csname urlstyle\endcsname\relax
  \providecommand{\doi}[1]{doi: #1}\else
  \providecommand{\doi}{doi: \begingroup \urlstyle{rm}\Url}\fi

\bibitem[Brock et~al.(2021)Brock, De, Smith, and Simonyan]{brock2021high}
Andrew Brock, Soham De, Samuel~L Smith, and Karen Simonyan.
\newblock {High-performance large-scale image recognition without
  normalization}.
\newblock arXiv:2102.06171, 2021.

\bibitem[Brown et~al.(2020)Brown, Mann, Ryder, Subbiah, Kaplan, Dhariwal,
  Neelakantan, Shyam, Sastry, Askell, et~al.]{brown2020language}
Tom~B Brown, Benjamin Mann, Nick Ryder, Melanie Subbiah, Jared Kaplan, Prafulla
  Dhariwal, Arvind Neelakantan, Pranav Shyam, Girish Sastry, Amanda Askell,
  et~al.
\newblock Language models are few-shot learners.
\newblock \emph{arXiv preprint arXiv:2005.14165}, 2020.

\bibitem[Senior et~al.(2020)Senior, Evans, Jumper, Kirkpatrick, Sifre, Green,
  Qin, {\v{Z}}{\'\i}dek, Nelson, Bridgland, et~al.]{senior2020improved}
Andrew~W Senior, Richard Evans, John Jumper, James Kirkpatrick, Laurent Sifre,
  Tim Green, Chongli Qin, Augustin {\v{Z}}{\'\i}dek, Alexander~WR Nelson, Alex
  Bridgland, et~al.
\newblock Improved protein structure prediction using potentials from deep
  learning.
\newblock \emph{Nature}, 577\penalty0 (7792):\penalty0 706--710, 2020.

\bibitem[Schrittwieser et~al.(2020)Schrittwieser, Antonoglou, Hubert, Simonyan,
  Sifre, Schmitt, Guez, Lockhart, Hassabis, Graepel,
  et~al.]{schrittwieser2020mastering}
Julian Schrittwieser, Ioannis Antonoglou, Thomas Hubert, Karen Simonyan,
  Laurent Sifre, Simon Schmitt, Arthur Guez, Edward Lockhart, Demis Hassabis,
  Thore Graepel, et~al.
\newblock Mastering atari, go, chess and shogi by planning with a learned
  model.
\newblock \emph{Nature}, 588\penalty0 (7839):\penalty0 604--609, 2020.

\bibitem[Silver et~al.(2017)Silver, Schrittwieser, Simonyan, Antonoglou, Huang,
  Guez, Hubert, Baker, Lai, Bolton, et~al.]{silver2017mastering}
David Silver, Julian Schrittwieser, Karen Simonyan, Ioannis Antonoglou, Aja
  Huang, Arthur Guez, Thomas Hubert, Lucas Baker, Matthew Lai, Adrian Bolton,
  et~al.
\newblock Mastering the game of go without human knowledge.
\newblock \emph{nature}, 550\penalty0 (7676):\penalty0 354--359, 2017.

\bibitem[He et~al.(2016)He, Zhang, Ren, and Sun]{he2016deep}
Kaiming He, Xiangyu Zhang, Shaoqing Ren, and Jian Sun.
\newblock Deep residual learning for image recognition.
\newblock In \emph{Proceedings of the IEEE conference on computer vision and
  pattern recognition}, pages 770--778, 2016.

\bibitem[Hastie et~al.(2009)Hastie, Tibshirani, and
  Friedman]{hastie2009elements}
Trevor Hastie, Robert Tibshirani, and Jerome Friedman.
\newblock \emph{{The Elements of Statistical Learning: Data Mining, Inference,
  and Prediction}}.
\newblock Springer, 2 edition, 2009.

\bibitem[Shalev-Shwartz and Ben-David(2014)]{shalev2014understanding}
Shai Shalev-Shwartz and Shai Ben-David.
\newblock \emph{Understanding Machine Learning: From Theory to Algorithms}.
\newblock Cambridge University Press, 2014.

\bibitem[Belkin et~al.(2019)Belkin, Hsu, Ma, and Mandal]{belkin2019reconciling}
Mikhail Belkin, Daniel Hsu, Siyuan Ma, and Soumik Mandal.
\newblock {Reconciling modern machine-learning practice and the classical
  bias--variance trade-off}.
\newblock \emph{Proceedings of the National Academy of Sciences}, 116\penalty0
  (32):\penalty0 15849--15854, 2019.
\newblock Previously arXiv:1812.11118.

\bibitem[Belkin et~al.(2020)Belkin, Hsu, and Xu]{belkin2019two}
Mikhail Belkin, Daniel Hsu, and Ji~Xu.
\newblock {Two models of double descent for weak features}.
\newblock \emph{SIAM Journal on Mathematics of Data Science}, 2\penalty0
  (4):\penalty0 1167--1180, 2020.
\newblock Accessed from arXiv:1903.07571.

\bibitem[Gy{\"o}rfi et~al.(2002)Gy{\"o}rfi, Kohler, Krzyzak, and
  Walk]{gyorfi2002distribution}
L{\'a}szl{\'o} Gy{\"o}rfi, Michael Kohler, Adam Krzyzak, and Harro Walk.
\newblock \emph{A distribution-free theory of nonparametric regression},
  volume~1.
\newblock SPRINGER, 2002.

\bibitem[Mei and Montanari(2019)]{mei2019generalization}
Song Mei and Andrea Montanari.
\newblock {The generalization error of random features regression: Precise
  asymptotics and double descent curve}.
\newblock arXiv:1908.05355, 2019.

\bibitem[Derezinski et~al.(2020)Derezinski, Liang, and
  Mahoney]{derezinski2020exact}
Michal Derezinski, Feynman~T Liang, and Michael~W Mahoney.
\newblock {Exact expressions for double descent and implicit regularization via
  surrogate random design}.
\newblock In \emph{Advances in Neural Information Processing Systems [NeurIPS
  2020]}, 2020.

\bibitem[Rangamani et~al.(2020)Rangamani, Rosasco, and
  Poggio]{rangamani2020interpolating}
Akshay Rangamani, Lorenzo Rosasco, and Tomaso Poggio.
\newblock For interpolating kernel machines, minimizing the norm of the erm
  solution minimizes stability, 2020.

\bibitem[Hastie et~al.(2019)Hastie, Montanari, Rosset, and
  Tibshirani]{hastie2019surprises}
Trevor Hastie, Andrea Montanari, Saharon Rosset, and Ryan~J Tibshirani.
\newblock Surprises in high-dimensional ridgeless least squares interpolation.
\newblock arXiv:1903.08560, 2019.

\bibitem[Adlam and Pennington(2020)]{adlam2020understanding}
Ben Adlam and Jeffrey Pennington.
\newblock Understanding double descent requires a fine-grained bias-variance
  decomposition.
\newblock In \emph{Advances in Neural Information Processing Systems [NeurIPS
  2020]}, 2020.

\bibitem[Bartlett et~al.(2021)Bartlett, Montanari, and
  Rakhlin]{bartlett2021deep}
Peter~L Bartlett, Andrea Montanari, and Alexander Rakhlin.
\newblock Deep learning: a statistical viewpoint.
\newblock arXiv:2103.09177, 2021.

\bibitem[Nakkiran et~al.(2019)Nakkiran, Kaplun, Bansal, Yang, Barak, and
  Sutskever]{nakkiran2019deep}
Preetum Nakkiran, Gal Kaplun, Yamini Bansal, Tristan Yang, Boaz Barak, and Ilya
  Sutskever.
\newblock Deep double descent: Where bigger models and more data hurt.
\newblock In \emph{International Conference on Learning Representations}, 2019.

\bibitem[Shawe-Taylor et~al.(2005)Shawe-Taylor, Williams, Cristianini, and
  Kandola]{shawe2005eigenspectrum}
John Shawe-Taylor, Christopher~KI Williams, Nello Cristianini, and Jaz Kandola.
\newblock On the eigenspectrum of the gram matrix and the generalization error
  of kernel-pca.
\newblock \emph{IEEE Transactions on Information Theory}, 51\penalty0
  (7):\penalty0 2510--2522, 2005.

\bibitem[Bousquet and Elisseeff(2002)]{bousquet2002stability}
Olivier Bousquet and Andr{\'e} Elisseeff.
\newblock Stability and generalization.
\newblock \emph{Journal of Machine Learning Research}, 2:\penalty0 499--526,
  2002.

\bibitem[Vershynin(2012)]{vershynin2010}
Roman Vershynin.
\newblock {Introduction to the non-asymptotic analysis of random matrices}.
\newblock In \emph{Compressed Sensing, Theory and Applications}, pages
  210--268. Cambridge University Press, 2012.
\newblock Accessed from arXiv:1011.3027.

\bibitem[Bai and Yin(1993)]{bai1993limit}
Zhi-Dong Bai and Yong-Qua Yin.
\newblock {Limit of the smallest eigenvalue of a large dimensional sample
  covariance matrix}.
\newblock \emph{The Annals of Probability}, 21\penalty0 (3):\penalty0
  1275--1294, 1993.

\bibitem[Nocedal and Wright(2006)]{NoceWrig06}
Jorge Nocedal and Stephen~J. Wright.
\newblock \emph{Numerical Optimization}.
\newblock Springer, New York, NY, USA, second edition, 2006.

\bibitem[Poggio et~al.(2019)Poggio, Kur, and Banburski]{poggio2019double}
Tomaso Poggio, Gil Kur, and Andrzej Banburski.
\newblock {Double descent in the condition number}.
\newblock Technical Report CBMM Memo No. 102, MIT, 2019.
\newblock Accessed from arXiv:1912.06190.

\bibitem[LeCun et~al.(1998)LeCun, Bottou, Orr, and Müller]{lecun98b}
Yann LeCun, Léon Bottou, Genevieve~B. Orr, and Klaus-Robert Müller.
\newblock Efficient backprop.
\newblock In \emph{Neural Networks: Tricks of the Trade (2nd ed.)}, Lecture
  Notes in Computer Science, pages 9--48. Springer, 1998.

\bibitem[Krizhevsky(2009)]{Krizhevsky09learningmultiple}
Alex Krizhevsky.
\newblock Learning multiple layers of features from tiny images.
\newblock Technical report, 2009.

\bibitem[Ioffe and Szegedy(2015)]{Ioffe15}
Sergey Ioffe and Christian Szegedy.
\newblock Batch normalization: Accelerating deep network training by reducing
  internal covariate shift.
\newblock In \emph{Proceedings of the 32nd International Conference on
  International Conference on Machine Learning - Volume 37}, ICML'15, page
  448–456. JMLR.org, 2015.

\bibitem[Ba et~al.(2016)Ba, Kiros, and Hinton]{Ba16}
Lei~Jimmy Ba, Jamie~Ryan Kiros, and Geoffrey~E. Hinton.
\newblock Layer normalization.
\newblock \emph{CoRR}, abs/1607.06450, 2016.

\bibitem[Li et~al.(2018)Li, Xu, Taylor, Studer, and Goldstein]{Li18}
Hao Li, Zheng Xu, Gavin Taylor, Christoph Studer, and Tom Goldstein.
\newblock Visualizing the loss landscape of neural nets.
\newblock In S.~Bengio, H.~Wallach, H.~Larochelle, K.~Grauman, N.~Cesa-Bianchi,
  and R.~Garnett, editors, \emph{Advances in Neural Information Processing
  Systems}, volume~31. Curran Associates, Inc., 2018.

\bibitem[De and Smith(2020)]{De2020}
Soham De and Sam Smith.
\newblock Batch normalization biases residual blocks towards the identity
  function in deep networks.
\newblock In H.~Larochelle, M.~Ranzato, R.~Hadsell, M.~F. Balcan, and H.~Lin,
  editors, \emph{Advances in Neural Information Processing Systems}, volume~33,
  pages 19964--19975. Curran Associates, Inc., 2020.

\end{thebibliography}

\newpage
\appendix

\section{Definitions from Linear Analysis}
\label{app:sec:defs}

We denote column vectors and matrices with small and capital bold letters, respectively, e.g. $\balpha=[\alpha_1, \alpha_2, \ldots, \alpha_d]\tp \in \R^d~$
and $\bA \in \R^{d_1 \times d_2 }$.
Singular values of a rectangular matrix $\bA \in \R^{n \times d}$ are denoted by $\smax(\bA) = s_1(\bA) \geq \ldots \geq s_{n \wedge d}(\bA) = \smin(\bA)$. The rank of $\bA$ is $r = \max\{ k \mid s_k(\bA) > 0 \}$.
Eigenvalues of a \acf{PSD} matrix $\bM \in \R^{d \times d}$ are nonnegative and are denoted 
$\lmax(\bM) = \lambda_1(\bM) \geq \ldots \geq \lambda_d(\bM) = \lmin(\bM)$,
while the smallest \emph{non-zero} eigenvalue is denoted $\lmin^+(\bM)$.

When $\bM \in \R^{d \times d}$ is positive definite, we define $\|\bx\|_{\bM}$ for $\bx \in \R^d$ by
$
\|\bx\|_{\bM} = \sqrt{\bx\tp \bM \bx}~.
$
It is easy to check that $\|\cdot\|_{\bM}$ is indeed a norm on $\R^d$, hence it induces a metric over $\R^d$, with the distance between $\bx$ and $\by$ given by $\|\bx - \by\|_{\bM} = \sqrt{(\bx - \by)\tp \bM (\bx - \by)}$. If $\bM$ is only semi-definite, these definitions would give a semi-norm and semi-metric.
Note that $\| \bx \|_{\bM} = \| \bM^{1/2}\bx \|$ where $\bM^{1/2}$ is the matrix square root of $\bM$. If we set $\bM = \bI$, the identity matrix, then the norm $\| \cdot \|_{\bM}$ reduces to the standard Euclidean norm: $\| \bx \| = \sqrt{\bx\tp \bx}$.

Combining the Cauchy-Schwarz inequality and the definition of operator norm $\|\bM \| = \smax(\bM)$, which implies $\| \bM \bx \| \leq \| \bM \| \| \bx \|$, we get the inequality $\| \bx \|_{\bM}^2 \leq \| \bx \|^2 \| \bM \|$.

The distance from a point $\bx$ to some set $B \subseteq \reals^d$ is defined as usual
\[
  \rho(\bx, B) = \inf_{\by \in B} \|\bx - \by\|~
\]
and for a positive defninte matrix $\bM$ we define similarly
\[
  \rho_{\bM}(\bx, B) = \inf_{\by \in B} \|\bx - \by\|_{\bM}~.
\]

Given $\bx_0 \in \reals^d$ and $\eps > 0$,
the Euclidean ball of radius $\eps$ centered at $\bx_0$ is defined as
\[
\sB(\bx_0, \eps) = \left\{ \bx \in \reals^d \mid \|\bx - \bx_0\| \leq \eps \right\} 
\]
and for a positive definite matrix $\bM$ the ellipsoid w.r.t.\ metric $\|\cdot\|_{\bM}$ is defined as
\[
\sE_{\bM}(\bx_0, \eps) = \left\{ \bx \in \reals^d \ \middle| \ \|\bx - \bx_0\|_{\bM} \leq \eps \right\} 
\]

\section{Minimum eigenvalue and condition number}\label{app:condition}

Previous works considered the link between the condition number of the features and the DD behavior~\cite{rangamani2020interpolating}. In this work, the analysis focuses more particularly on the minimum eigenvalue. In the following small experiments, we empirically show that in the experiments shown in the main paper, the condition number is driven by the minimum eigenvalue, and that the maximum eigenvalue stays close to a constant order when we increase the size of the features. In Figure~\ref{fig:condition}, we use the same setting as in the MNIST experiment in Figure~\ref{fig:MNIST_FashionMNIST} is the main paper. We obseve that the behavior of the condition number follows the minimum eigenvalue, while the maximum eigenvalue stays between 10 and 100 as we increase the width of the networks.

\begin{figure}[H]
  \centering
  \includegraphics[width=.95\textwidth]{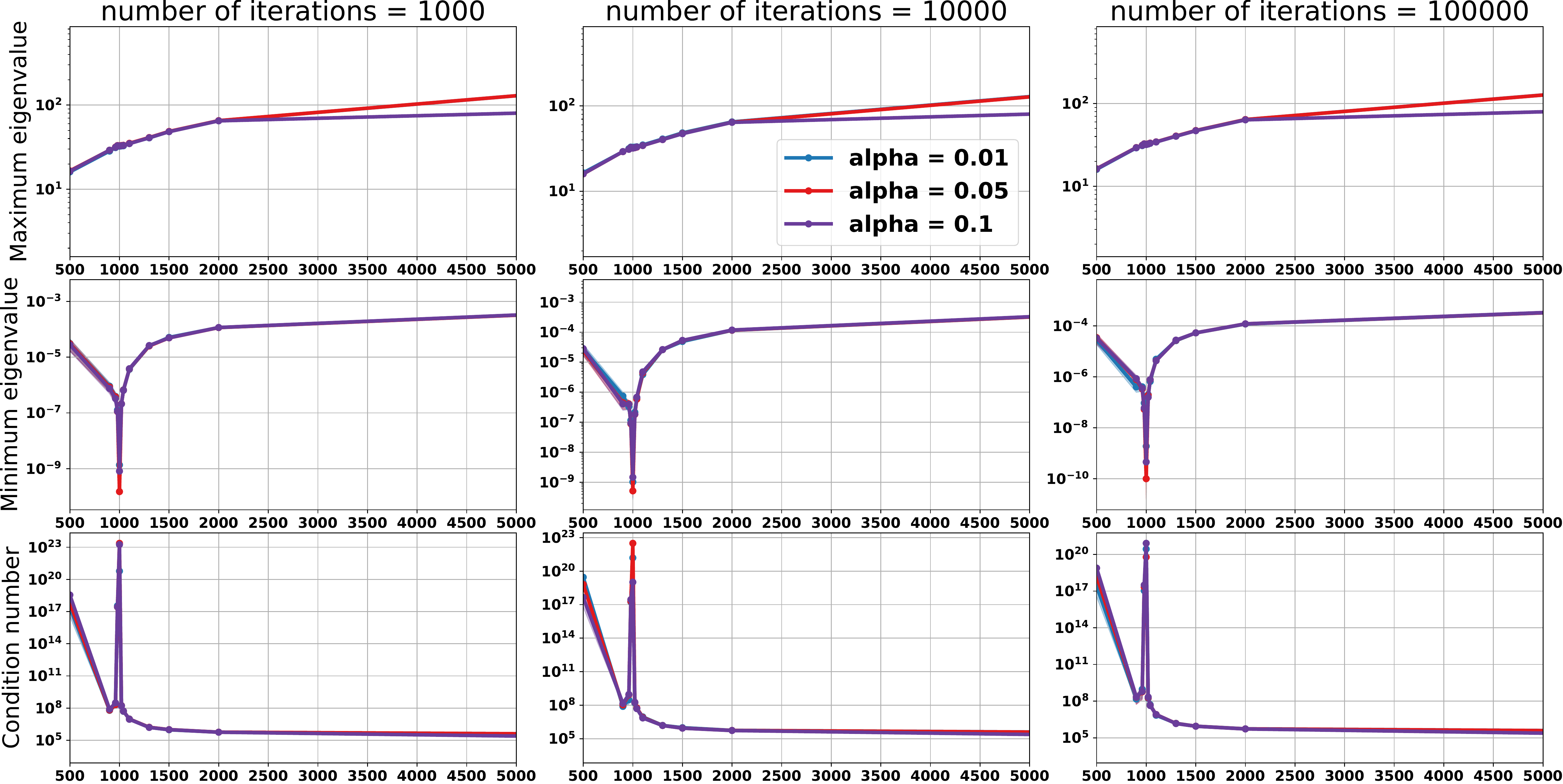}
  \caption{Maximum and minimum eigenvalues and condition numbers of the features of one hidden layer networks of variable width: MNIST - 1000 samples for training, networks trained with gradient descent and different step sizes.}
  \label{fig:condition}
\end{figure}

\section{Excess Risk of Gradient Descent}
\label{app:excess_risk_gd}
In this section we consider the standard \ac{GD} algorithm, that is $\alg(\bw_0) = \bw_T$, which is obtained recursively by applying the update rule $\bw_{t+1} = \bw_t - \alpha \nabla \Lh_S(\bw_t)$ with some step size $\alpha > 0$ and initialization $\bw_0 \in \reals^d$. The rule is iterated  for $t=0, \ldots, T-1$.
We will pay attention to $\alg$ which satisfy the following regularity condition:
\setcounter{definition}{0}
\begin{definition}
  A map $f : \reals^d \to \reals^d$ is called $(\Delta, \bM)$-admissible, where $\bM$ is a fixed \ac{PSD} matrix and $\Delta \geq 0$, 
  if for all $\bw, \bw' \in \reals^d$ the following holds:
  \[
    \|f(\bw) - f(\bw')\|_{\bM} \leq \Lip \|\bw - \bw'\|~.
  \]
\end{definition}
Notice that the norm on the left-hand side is $\|\cdot\|_{\bM}$, while that on the right-hand side is the standard Euclidean norm.
Also notice that this inequality entails a Lipschitz condition with Lipschitz factor $\Delta$.

The \emph{excess risk} of $\alg(\bW_0)$ is defined as
\[
\sE(\bwstar) = L(\alg(\bW_0)) - L(\bwstar) \qquad \bwstar \in \argmin_{\bw \in \reals^d} L(\bw)~.
\]
Next we give upper bounds on the excess risk of \ac{GD} output, assuming that the output of $\alg$ is of \emph{low-rank}, which is of interest in the overparameterized regime.
Specifically, for some low-rank orthogonal projection $\bM \in \reals^{d \times d}$
we assume that $\bM \alg(\bw) = \alg(\bw)$ a.s. with respect to random samples $S$, for any initialization $\bw$.
The following theorem gives us a general bound on the excess risk of any admissible algorithm in a sense of \cref{def:admissible} w.r.t.\ to any smooth loss (not necessarily convex).
In the following we will demonstrate that \ac{GD} satisfies \cref{def:admissible}.
\setcounter{theorem}{0}
\begin{theorem}[Excess Risk of Admissible Algorithm]
  Assume that $\bW_0 \sim \sN(\bzero, \initvar\bI_{d\times d})$,
  and assume that $\alg$ is $(\Delta, \bM)$-admissible (\cref{def:admissible}), where $\Delta$ and $\bW_0$ are independent.
  Further assume $\bM \alg(\bw) = \alg(\bw)$ for any $\bw$, 
  and that $L$ and $\Lh$ are $H$-smooth.
  Then, for any $\bwstar \in \argmin_{\bw \in \reals^d} L(\bw)$ we have
  \[
    \E[\sE(\bwstar)]
    \leq
    H \pr{
      \E[\Delta^2] \pr{ \|\bwstar\|^2 + \initvar (2 + d) }
      +
      \E[\|\alg(\bwstar) - \bwstar\|_{\bM}^2]
      +
      \frac12 \E[\|\bwstar\|^2_{\bI-\bM}]
    }~.
  \]
  In particular, having $\alpha \leq 1/H$,
  \begin{align*}
    \E[\|\alg(\bwstar) - \bwstar\|_{\bM}^2]
    \leq 2 \alpha T L(\bwstar)~.
  \end{align*}
\end{theorem}
\begin{proof}
By the $H$-smoothness of $L$, and noting that $\nabla L(\bwstar) = 0$ since $\bwstar$ is a minimizer,
\begin{align*}
  L(\alg(\bW_0)) - L(\bwstar)
  &\leq
    \frac{H}{2} \|\alg(\bW_0) - \bwstar\|^2\\
  &=
    \frac{H}{2} \|\alg(\bW_0) - \bwstar\|^2_{\bM} + \frac{H}{2} \|\alg(\bW_0) - \bwstar\|^2_{\bI-\bM}\\
  &=
    \frac{H}{2} \|\alg(\bW_0) - \bwstar\|^2_{\bM} + \frac{H}{2} \|\bwstar\|^2_{\bI-\bM}
\end{align*}
where the last equality is justified since $\bM$ is an orthogonal projection satisfying the assumption $\bM \alg(\bW_0) = \alg(\bW_0)$.

Focusing on the first term above, we get
\begin{align*}
  \|\alg(\bW_0) - \bwstar\|_{\bM}^2
  &\leq
    2 \|\alg(\bW_0) - \alg(\bwstar)\|_{\bM}^2 + 2 \|\alg(\bwstar) - \bwstar\|_{\bM}^2\\
  &\leq
    2 \Delta^2 \|\bW_0 - \bwstar\|^2 + 2 \|\alg(\bwstar) - \bwstar\|_{\bM}^2
\end{align*}
where the first inequality is due to the following inequality for squared Euclidean norms: $\| \ba_1 + \cdots + \ba_n \|^2 \leq n(\|\ba_1\|^2 + \cdots + \|\ba_n\|^2)$, and the last inequality is due to $(\Delta, \bM)$-admissibility of $\alg$.
Taking expectation on both sides we have
\begin{align*}
  \E\br{\|\alg(\bW_0) - \bwstar\|_{\bM}^2}
  \leq
  2 \E[\Delta^2] \E[\|\bW_0 - \bwstar\|^2] + 2 \E[\|\alg(\bwstar) - \bwstar\|_{\bM}^2]
\end{align*}
where $\Delta$ is random only due to the sample (recall that $\Delta$ is independent of $\bW_0$ by assumption).
The term $\E[\|\bW_0 - \bwstar\|^2]$ is a standard Gaussian integral, whose calculation is summarized in following lemma:
\begin{lemma}[Expectation of a squared norm of the Gaussian random vector]
  \label{lem:gaussian_norm}
  For any $\nu > 0$ and $\bx_0 \in \reals^d$:
  \[
    \frac1Z \int_{\reals^d} \|\bx - \bx_0\|^2 e^{-\frac{1}{2 \nu^2} \|\bx\|^2} \diff \bx
    =
    \|\bx_0\|^2 + \nu^2 (2 + d)
  \]
  where $Z = \int_{\reals^d} e^{-\frac{1}{2 \nu^2} \|\bx\|^2} \diff \bx$ is the normalization constant. Therefore, if $\bX \sim \sN(\bzero, \nu^2\bI_{d\times d})$ then for any $\bx_0 \in \reals^d$ we have the identity $\E[\|\bX - \bx_0\|^2] = \|\bx_0\|^2 + \nu^2 (2 + d)$.
\end{lemma}
\begin{proof}[Proof of \cref{lem:gaussian_norm}]
  \begin{align*}
    \frac1Z \int_{\reals^d} \|\bx - \bx_0\|^2 e^{-\frac{1}{2 \nu^2} \|\bx\|^2} \diff \bx
    &=
    \|\bx_0\|^2
    +
      \frac1Z \int_{\reals^d} \|\bx\|^2 e^{-\frac{1}{2 \nu^2} \|\bx\|^2} \diff \bx\\
    &=
      \|\bx_0\|^2
      +
      \nu^2 \E[\|\bX\|^2]
      =
      \|\bx_0\|^2 + 2 \nu^2 \frac{\Gamma\pr{\frac{d + 2}{2}}}{\Gamma\pr{\frac{d}{2}}}\\
    &=
      \|\bx_0\|^2 + \nu^2 (2 + d)~,
  \end{align*}
  since $\|\bX\|^2$ is $\chi^2$-distributed with $d$ degrees of freedom.
\end{proof}
For our case with $\bW_0 \sim \sN(\bzero, \initvar\bI_{d\times d})$ this gives $\E[\|\bW_0 - \bwstar\|^2] \leq \|\bwstar\|^2 + \initvar (2 + d)$.

The term $\E[\|\alg(\bwstar) - \bwstar\|_{\bM}^2]$ is bounded next by using the standard ``descent lemma''.
\begin{lemma}[Descent Lemma]
\label{lem:descent_lemma}
  Assuming that $\alpha \leq 1/H$,
  \[
    \sum_{t=0}^{T-1} \|\nabla\Lh_S(\bw_t)\|^2 \leq \frac{2}{\alpha}\pr{\Lh_S(\bw_0) - \Lh_S(\bw_T)}
  \]
\end{lemma}
\begin{proof}[Proof of \cref{lem:descent_lemma}]
  Since $\Lh$ is $H$-smooth, Taylor expansion and the gradient descent rule give us
  \begin{align*}
    \Lh_S(\bw_{t+1}) - \Lh_S(\bw_t) 
    \leq -\alpha \|\nabla \Lh_S(\bw_t)\|^2 + \frac{\alpha^2 H}{2} \|\nabla \Lh_S(\bw_t)\|^2~,
  \end{align*}
  and rearranging we have
  \begin{align*}
    \pr{\alpha - \frac{\alpha^2 H}{2}} \|\nabla \Lh_S(\bw_t)\|^2 \leq \Lh_S(\bw_t) - \Lh_S(\bw_{t+1})~.
  \end{align*}
  Summing over $t=0,\ldots,T-1$ we arrive at
  \begin{align*}
    \pr{\alpha - \frac{\alpha^2 H}{2}} \sum_{t=0}^{T-1} \|\nabla \Lh_S(\bw_t)\|^2 \leq \Lh_S(\bw_0) - \Lh_S(\bw_T)~.
  \end{align*}
  Finally, note that $\alpha - \frac{\alpha^2 H}{2} > \frac{\alpha}{2}$ by the assumption that $\alpha \leq 1/H$.
\end{proof}
In particular, if $\bwstar_t$ are the iterates of \ac{GD} when starting from $\bwstar$ (so that $\bwstar_0 = \bwstar$), then
\begin{align*}
  \|\alg(\bwstar) - \bwstar\|_{\bM}^2
  &= \lf\|\alpha\sum_{t=0}^{T-1} \nabla \Lh_S(\bwstar_t)\rt\|_{\bM}^2\\
  &\leq \alpha^2 T \sum_{t=0}^{T-1} \|\nabla \Lh_S(\bwstar_t)\|_{\bM}^2\\
  &\leq \alpha^2 \frac{2 T}{\alpha} \pr{\Lh_S(\bwstar) - \Lh_S(\bwstar_T)}\\
  &\leq 2\alpha T \Lh_S(\bwstar)
\end{align*}

and taking expectation on both sides we have
\[
  \E[\|\alg(\bwstar) - \bwstar\|_{\bM}^2] 
  \leq 2\alpha T L(\bwstar)~.
\]
Putting all together completes the proof of \cref{thm:excess_risk}.
\end{proof}

\subsection{Least-Squares with Random Design and without Label Noise}
\label{app:sec:ls_random_design_noiseless}
Consider a noise-free linear regression model
\[
  Y = \bX\tp \bwstar~,
\]
where instances are distributed according to some unknown distribution $P_X$ supported on a $d$-dimensional unit Euclidean ball.
After observing a training sample $S = \pr{(\bX_i, Y_i)}_{i=1}^n$, we run \ac{GD} on the given empirical square loss
\[
  \Lh_S(\bw) = \frac{1}{2 n} \sum_{i=1}^n (\bw\tp \bX_i - Y_i)^2~.
\]
Let a sample covariance matrix be defined as $\bhSigma = (\bX_1 \bX_1\tp + \dots + \bX_n \bX_n\tp) / n$, and
let $\bhSigma = \bU \bS \bV\tp$ be the \ac{SVD} of $\bhSigma$.
We will use a subscript notation $\bU_r = [\bu_1, \ldots, \bu_r]$, $\bV_r = [\bv_1, \ldots, \bv_r]$, and $S_r = \text{diag}(s_1(\bhSigma), \ldots, s_r(\bhSigma))$, where $r = \rank(\bhSigma)$ to indicate non-degenerate orthonormal bases and their scaling matrix.
In the setting of our interest $\bhSigma$ might be degenerate, and therefore we will occasionally refer to the non-degenerate subspace $\bU_r$.
We write $\lminh^+ = \lmin^+(\bhSigma) = \lambda_r(\bhSigma)$ for the minimal \emph{non-zero} eigenvalue, and
we denote $\bhatM = \bU_r \bU_r\tp$.
Note that $\bhatM^2 = \bhatM$.
Now we state the main result of this section.

\begin{theorem}
  Assume that $\bW_0 \sim \sN(\bzero, \initvar \bI)$.
  Then, for any $\bwstar \in \argmin_{\bw \in \reals^d} L(\bw)$,
  \[
    \E[\sE(\bwstar)]
    \leq
      \E\br{(1 - \alpha \lminh^+)^{2T}} \pr{ \|\bwstar\|^2 + \initvar (2 + d) }
      +
      \frac12 \E[\|\bwstar\|^2_{\bI-\bhatM}]~.
  \]
\end{theorem}
\begin{proof}
The proof is a consequence of \cref{thm:excess_risk}, modulo showing that \ac{GD} with the least-squares objective is $((1-\alpha \lminh^+)^T, \bhatM)$-admissible (\cref{prop:A_Lip} below).
\begin{prop}
  \label{prop:OLS_GD_map}
  For a $T$-step gradient descent map $\alg : \reals^d \to \reals^d$ with step size $\alpha > 0$ applied to the least-squares, and for all $\bw_0 \in \reals^d$, we have a.s.\ that
  \begin{align*}
    \alg(\bw_0) = (\bI - \alpha \bhSigma)^T \bw_0 + \alpha \sum_{t=0}^{T-1} (\bI - \alpha \bhSigma)^t \pr{\frac1n \sum_{i=1}^n \bX_i Y_i}~.
  \end{align*}
\end{prop}
\begin{proof}[Proof of \cref{prop:OLS_GD_map}]
Abbreviate $\bC = (\bX_1 Y_1 + \cdots + \bX_n Y_n)/n$.
Since $\nabla \wh{L}_S(\bw) = \bhSigma \bw - \bC$, observe that
\begin{align*}
  \bw_{t}
  =
    \bw_{t-1} - \alpha (\bhSigma \bw_{t-1} - \bC)
  =
    (\bI - \alpha \bhSigma) \bw_{t-1} + \alpha \bC~.
\end{align*}
A simple recursive argument reveals that for every $\bw_0 \in \reals^d$
\begin{align*}
  \alg(\bw_0) = \bw_T
  &=
    (\bI - \alpha \bhSigma) \bw_{T-1} + \alpha \bC\\
  &=
    (\bI - \alpha \bhSigma)^2 \bw_{T-2} + \alpha (\bI - \alpha \bhSigma) \bC + \alpha \bC\\
  &=
    (\bI - \alpha \bhSigma)^3 \bw_{T-3} + \alpha (\bI - \alpha \bhSigma)^2 \bC + \alpha (\bI - \alpha \bhSigma) \bC + \alpha \bC\\
  &\cdots\\
  &=
    (\bI - \alpha \bhSigma)^T \bw_0 + \alpha \sum_{t=0}^{T-1} (\bI - \alpha \bhSigma)^t \bC~.
\end{align*}
\end{proof}
\cref{prop:OLS_GD_map} implies the following simple fact.
\begin{cor}[Admissibility of \ac{GD}]
  \label{prop:A_Lip}
  The $T$-step gradient descent map $\alg : \reals^d \to \reals^d$ with step size $\alpha > 0$ applied to the least-squares problem satisfies, for all $\bw_0, \bu_0 \in \reals^d$,
  \begin{align*}
    \|\alg(\bw_0) - \alg(\bu_0)\|_{\bhatM}
    \leq
    (1 - \alpha \lminh^+)^T \|\bw_0 - \bu_0\|~.
  \end{align*}
\end{cor}
\begin{proof}[Proof of \cref{prop:A_Lip}]
  By~\cref{prop:OLS_GD_map} for any $\bw_0, \bu_0 \in \reals^d$:
  \begin{align*}
    \|\alg(\bw_0) - \alg(\bu_0)\|_{\bU_r \bU_r\tp}
    &=
      \|(\bI - \alpha \bhSigma)^T (\bw_0 - \bu_0)\|_{\bU_r \bU_r\tp}\\
    &=
      \|\bU_r\tp(\bI - \alpha \bhSigma)^T (\bw_0 - \bu_0)\|\\
    &\leq
      \|\bU_r\tp (\bI - \alpha \bhSigma)^T\| \| (\bw_0 - \bu_0)\|~. 
  \end{align*}
  Now,
  \begin{align*}
    \bU_r\tp (\bI - \alpha \bhSigma)^T
    =
    \bU_r\tp \bU (\bI - \alpha \bS)^T \bV\tp
    =
    \bI_{r \times d} (\bI - \alpha \bS)^T \bV\tp
    =
    (\bI_{r \times d} - \alpha \bS_{r \times d})^T \bV\tp
  \end{align*}
  where subscript $r \times n$ stands for clipping the matrix to $r$ rows and $d$ columns.
  The above implies that the operator norm of $\bU_r\tp (\bI - \alpha \bhSigma)^T$ satisfies $\|\bU_r\tp (\bI - \alpha \bhSigma)^T\| \leq (1-\alpha \lmin^+(\bhSigma))^T$.
\end{proof}
Finally, note that in the overparametrized case ($d>n$) we have $r = n \wedge d = n$.
\end{proof}
\subsection{Least-Squares with Random Design and Label Noise}
\label{app:sec:ls_random_design_withnoise}
Now, in addition to the random design we introduce a label noise into our model:
\[
  Y = \bX\tp \bwstar + \ve~,
\]
where we have independent noise $\ve$ such that $\E[\ve] = 0$ and $\E[\ve^2] = \noisevar$.

\begin{theorem}
  Assume that $\bW_0 \sim \sN(\bzero, \initvar \bI)$.
  Then, for any $\bwstar \in \argmin_{\bw \in \reals^d} L(\bw)$,
  \[
    \E[\sE(\bwstar)]
    \leq
      \E\br{(1 - \alpha \lminh^+)^{2T}} \pr{ \|\bwstar\|^2 + \initvar (2 + d) }
      +
      \frac{4 \sigma^2}{n} \E\br{\pr{\lminh^+}^{-2}}
      +
      \frac12 \E[\|\bwstar\|^2_{\bI-\bhatM}]~.
  \]
\end{theorem}
\begin{proof}
The proof is almost identical to the one of \cref{thm:excess_ls_noiseless} except $\E\br{ \|\bwstar - \sA_S(\bwstar)\|_{\bhatM}^2 }$ is handled by the following \cref{lem:ls_noise}.
\end{proof}

\begin{lemma}
\label{lem:ls_noise}
  Let $\bhatM$ be defined as in \cref{sec:ls_random_design}.
  For any $T > 0$, \ac{GD} achieves
  \begin{align*}
    \E\br{ \|\bwstar - \sA_S(\bwstar)\|_{\bhatM}^2 }
    \leq
    \frac{4 \sigma^2}{n} \E\br{\pr{\lminh^+}^{-2}}~.
  \end{align*}
\end{lemma}
\begin{proof}
  Recall that $\E[\ve_i] = 0$ and $\E[\ve_i^2] = \sigma^2$ for $i \in [n]$.
  Throughout the proof abbreviate $\E[\cdot ~|~ \bX_1, \ldots, \bX_n] = \E_{\bve}[\cdot]$.

  We begin by noting that the integral form of Taylor theorem gives us that for any $\bwstar \in \argmin_{\bw \in \reals^d} \Lh(\bw)$ and any $\bw \in \reals^d$,
\begin{align*}
  \Lh(\bw) - \Lh(\bwstar)
  &=
    \frac12 (\bw - \bwstar)\tp \pr{\int_0^1 \nabla^2 \Lh(\tau \bw + (1-\tau) \bwstar) \diff \tau} (\bw - \bwstar)\\
  &\geq
    \frac12 \cdot \lminh^+ (\bw - \bwstar)\tp \bhatM (\bw - \bwstar)~.
\end{align*}
Thus, taking $\bw = \sA_S(\bwstar)$, we have
\begin{align*}
  \E_{\bve}\br{\|\bwstar - \sA_S(\bwstar)\|_{\bhatM}^2}
  &\leq \frac{1}{\lminh^+} \pr{\E_{\bve}\Lh(\bwstar) - \E_{\bve}\br{\Lh(\sA_S(\bwstar))}}\\
  &= \frac{1}{\lminh^+} \pr{\sigma^2 - \E_{\bve}\br{\Lh(\sA_S(\bwstar))}}~.
\end{align*}
Now, let's focus on the loss term on the r.h.s.:
\begin{align*}
  \E_{\bve}\br{\Lh_S(\bwstar_T)}
  &=
    \frac1n \sum_{i=1}^n \E_{\bve}\br{\pr{\pr{\bwstar_T - \bwstar_0}\tp \bX_i - \ve_i}^2}\\
  &=
    \sigma^2
    -
    \frac2n \sum_{i=1}^n \E_{\bve}\br{\ve_i \pr{\bwstar_T - \bwstar_0}\tp \bX_i}
    +
    \E_{\bve}\br{\pr{\bwstar_T - \bwstar_0}\tp \bhSigma \pr{\bwstar_T - \bwstar_0}}\\
  &\geq
    \sigma^2
    -
    \frac2n \sum_{i=1}^n \E_{\bve}\br{\ve_i \pr{\bwstar_T - \bwstar_0}\tp \bX_i}\\
  &=
    \sigma^2
    -
    \frac2n \sum_{i=1}^n \E_{\bve}\br{\ve_i {\bwstar_T}\tp \bX_i}
\end{align*}
where the last term is small when label noise is not too correlated with the output $\bwstar_T$.
Hence to control the term, we need to measure the effect of the noise on \ac{GD}.
To do so we will introduce an additional iterates $(\btilw_t)_t$ constructed by running \ac{GD} on labels without noise, that is
\begin{align*}
  \btilwstar_{t+1} = \btilwstar_t - \alpha \nabla \tilde{L}_S(\btilwstar_t)
  \qquad \text{where} \quad
  \tilde{L}(\bw) = \frac{1}{2n} \sum_{i=1}^n \pr{\bw\tp \bX_i - {\bwstar}\tp \bX_i}^2~.
\end{align*}
The plan is then to bound the deviation $\|\bwstar_T - \btilwstar_T\|_{\bhatM}$ which we will do recursively.
We proceed:
\begin{align*}
  &\frac2n \sum_{i=1}^n \E_{\bve}\br{\ve_i {\bwstar_T}\tp \bX_i}\\
  &=
    \frac2n \sum_{i=1}^n \E_{\bve}\br{\ve_i (\bwstar_T - \btilwstar_T)\tp \bX_i} \tag{Note that $\E_{\bve}[\btilwstar_T ~|~ \bX_i] = 0$}\\
  &=
    \frac2n \sum_{i=1}^n \E_{\bve}\br{\ve_i (\bwstar_T - \btilwstar_T)\tp \bhatM \bX_i} \tag{Since $\bhatM \bX_i = \bX_i$}\\
  &\leq
    \frac2n \E_{\bve}\br{\lf\|\sum_{i=1}^n \ve_i \bX_i \rt\| \lf\|\bhatM(\bwstar_T - \btilwstar_T)\rt\|} \tag{Cauchy-Schwarz}
\end{align*}
Now we will handle $\lf\|\bhatM(\bwstar_T - \btilwstar_T)\rt\| = \|\bwstar_T - \btilwstar_T\|_{\bhatM}$
by following a recursive argument.
First, observe that for any $t=0,1,2,\ldots$
\begin{align*}
  \nabla \Lh(\btilwstar_t)
  = \bhSigma \btilwstar_t - \frac1n \sum_{i=1}^n \bX_i \bX_i\tp \bwstar_0
  = \bhSigma (\btilwstar_t - \bwstar_0)~,
\end{align*}
and at the same time
\begin{align*}
  \nabla \Lh(\bwstar_t) = \bhSigma \bwstar_t - \frac1n \sum_{i=1}^n \bX_i \bX_i\tp \bwstar_0 - \frac1n \sum_{i=1}^n \bX_i \ve_i
  = \bhSigma (\bwstar_t - \bwstar_0) - \frac1n \sum_{i=1}^n \bX_i \ve_i~.
\end{align*}
Thus,
\begin{align}
  \|\bwstar_{t+1} - \btilwstar_{t+1}\|_{\bhatM}
  &=
    \lf\|\bwstar_t - \btilwstar_t - \alpha \pr{\nabla \Lh(\bwstar_t) - \nabla \Lh(\btilwstar_t)}\rt\|_{\bhatM} \label{eq:wstar_wtilstar_recursion_1}\\
  &=
    \lf\|\bwstar_t - \btilwstar_t - \alpha \bhSigma (\bwstar_t - \btilwstar_t) - \frac{\alpha}{n} \sum_{i=1}^n \bX_i \ve_i\rt\|_{\bhatM} \nonumber\\
  &=
    \lf\|(\bI - \alpha \bhSigma) (\bwstar_t - \btilwstar_t)\rt\|_{\bhatM} + \frac{\alpha}{n} \lf\|\sum_{i=1}^n \bX_i \ve_i\rt\|_{\bhatM} \nonumber\\
  &\stackrel{(a)}{\leq}
    \|\bI - \alpha \bhSigma\|_{\bhatM} \|\bwstar_t - \btilwstar_t\|_{\bhatM} + \frac{\alpha}{n} \lf\| \sum_{i=1}^n \bX_i \ve_i\rt\|_{\bhatM} \nonumber\\
  &\leq
    (1 - \alpha \lminh^+) \|\bwstar_t - \btilwstar_t\|_{\bhatM} + \frac{\alpha}{n} \lf\| \sum_{i=1}^n \bX_i \ve_i\rt\|_{\bhatM}~. \label{eq:wstar_wtilstar_recursion_2}
\end{align}
where in the step $(a)$ we note that
$\bhatM (\bI - \alpha \bhSigma) (\bwstar_t - \btilwstar_t) = \bhatM (\bI - \alpha \bhSigma) \bhatM (\bwstar_t - \btilwstar_t)$ (since $\bhatM^2 = \bhatM$ and $\bhSigma \bhatM = \bhSigma$).

Now we use the fact that an elementary recursive relation $x_{t+1} \leq a_t x_t + b_t$ with $x_0 = 0$ unwinds to
$x_T \leq \sum_{t=1}^T b_t \prod_{k=t+1}^T a_k$, which gives
\begin{align*}
  \|\bwstar_{T} - \btilwstar_{T}\|_{\bhatM}
  &\leq
    \frac{\alpha}{n} \lf\| \sum_{i=1}^n \bX_i \ve_i\rt\|_{\bhatM} \sum_{t=1}^T (1 - \alpha \lminh^+)^{T-t}\\
  &\leq
    \frac{\alpha}{n} \lf\| \sum_{i=1}^n \bX_i \ve_i\rt\|_{\bhatM} \frac{1 - (1 - \alpha \lminh^+)^T}{\alpha \lminh^+}~.
\end{align*}
Thus,
\begin{align*}
  \frac2n \sum_{i=1}^n \E_{\bve}\br{\ve_i \pr{\bwstar_T - \bwstar_0}\tp \bX_i}
  &\leq
        \frac{2}{n} \cdot \frac1n \E_{\bve}\br{ \lf\| \sum_{i=1}^n \bX_i \ve_i\rt\|^2 \frac{1}{\lminh^+} }\\
  &\leq
    \frac{2 \sigma^2}{n} \cdot \frac{1}{\lminh^+}
\end{align*}
where we used a basic fact that
\begin{align*}
    \E_{\bve}\br{\lf\|\sum_{i=1}^n \bX_i \ve_i\rt\|^2 \bmid \bX_1, \ldots, \bX_n}
  =
    \noisevar \sum_{i=1}^n \|\bX_i\|^2
    \leq \noisevar n~.
\end{align*}
Putting all together completes the proof.
\end{proof}
\subsection{Concentration of the Smallest Non-zero Eigenvalue}
\label{app:sec:lmin_concentration}
In this section we take a look at the behaviour of $\lminh^+$ assuming that training instances are now \emph{random} independent vectors $\bX_1, \ldots, \bX_n$ sampled i.i.d.\ from some underlying marginal density. Recall that the sample covariance matrix is $\bhSigma = (\bX_1 \bX_1\tp + \dots + \bX_n \bX_n\tp) / n$.
We focus on the concentration of $\lminh^+ = \lmin^+(\bhSigma)$ around its population counterpart $\lmin^+ = \lmin^+(\bSigma)$, where $\bSigma$ is the population covariance matrix: $\bSigma = \E[\bX_1 \bX_1\tp]$.
Note that defining $\bX = [\bX_1, \ldots, \bX_n] \in \reals^{d \times n}$ we have: $\bhSigma = \bX \bX\tp / n$.

In particular, we are concerned with a non-asymptotic version of the  \emph{Bai-Yin law}~\citep{bai1993limit}, which says that the smallest eigenvalue (for $d \leq n$), or the $(d-n+1)$-th smallest eigenvalue (for $d > n$), of a sample covariance matrix with independent entries has almost surely an asymptotic behavior $(1-\sqrt{d/n})^2$ as $n \rightarrow \infty$.
The setting $d > n$ is essential for our case, as it corresponds to overparametrization.
However, unlike \cite{bai1993limit}, we do not assume independence of entries, but rather independence of observations (columns of $\bX$).
This will be done by introducing a distributional assumption on observations: we assume that observations are \emph{sub-Gaussian}.
\begin{definition}[Sub-Gaussian random vectors]
  A random vector $\bX \in \reals^d$ is sub-Gaussian if the random variables $\bX\tp \by$ are sub-Gaussian for all $\by \in \reals^d$.
  The sub-Gaussian norm of a random vector $\bX \in \reals^d$ is defined as
  \[
    \|\bX\|_{\psi_2} = \sup_{\|\by\| = 1}\sup_{p \geq 1}\cbr{ \frac{1}{\sqrt{p}} \E[|\bX\tp \by|^p]^{\frac1p} }~.
  \]
\end{definition}
We will also require the following definition.
\begin{definition}[Isotropic random vectors]
  A random vector $\bX \in \reals^d$ is called isotropic if its covariance is the identity: $\E\br{\bX \bX\tp} = \bI$.
  Equivalently, $\bX$ is isotropic if $\E[(\bX\tp \bx)^2] = \|\bx\|^2$ for all $\bx \in \reals^d$.
\end{definition}
Let $\bSigma^{\pinv}$ be the Moore-Penrose pseudoinverse of $\bSigma$.
In \cref{sec:lmin_concentration_proof} we prove the following.
\setcounter{lemma}{0}
\begin{lemma}[]
  Let $\bX = [\bX_1, \ldots, \bX_n] \in \reals^{d \times n}$ be a matrix with i.i.d.\ columns, such that $\max_i\|\bX_i\|_{\psi_2} \leq K$,
  and let $\bhSigma = \bX \bX\tp / n$, and $\bSigma = \E[\bX_1 \bX_1\tp]$.
  Then, for every $x \geq 0$, with probability at least $1-2e^{-x}$, we have
  \[
    \lmin^+(\bhSigma)
    \geq
    \lmin^+(\bSigma) \pr{1 - K^2 \pr{c \sqrt{\frac{d}{n}} + \sqrt{\frac{x}{n}}}}_+^2
    \qquad
    \text{for } n \geq d~,
  \]
  and furthermore, assuming that $\|\bX_i\|_{\bSigma^{\pinv}} = \sqrt{d}$ \ a.s. for all $i \in [n]$, we have
  \[
    \lmin^+(\bhSigma)
    \geq
    \lmin^+(\bSigma) \pr{\sqrt{\frac{d}{n}} - K^2 \pr{c + 6 \sqrt{\frac{x}{n}}}}_+^2
    \qquad
    \text{for } n < d~,
  \]
  where we have an absolute constant $c = 2^{3.5} \sqrt{\ln(9)}$.
\end{lemma}
Next we present the proof of the Lemma.
\section{Concentration of the Smallest Non-zero Eigenvalue: Proof}
\label{sec:lmin_concentration_proof}
The next theorem gives us a non-asymptotic version of Bai-Yin law~\citep{bai1993limit} for rectangular matrices whose rows are sub-Gaussian isotropic random vectors.
\begin{theorem}[{\protect \cite[Theorem 5.39]{vershynin2010}}]
  \label{thm:matrix_concentration_indep_rows}
  Let $\bA \in \reals^{n \times d}$ whose rows $(\bA\tp)_i$ are independent sub-Gaussian isotropic random vectors in $\reals^d$, such that $K = \max_{i\in [n]}\|(\bA\tp)_i\|_{\psi_2}$.
  Then for every $x \geq 0$, with probability at least $1-2 e^{-x}$ one has
  \[
    \sqrt{n} - 2^{3.5} K^2 (\sqrt{\ln(9) d} + \sqrt{x})
    \leq
    \smin(\bA)
    \leq
    \smax(\bA)
    \leq
    \sqrt{n} + 2^{3.5} K^2 \pr{\sqrt{\ln(9) d} + \sqrt{x}}~.
  \]
\end{theorem}
\begin{theorem}[{\protect \cite[Theorem 5.58]{vershynin2010}}]
  \label{thm:matrix_concentration_indep_cols}
  Let $\bA \in \reals^{d \times n}$ whose columns $\bA_i$ are independent sub-Gaussian isotropic random vectors in $\reals^d$ with $\|\bA_i\| = \sqrt{d}$  a.s., such that $K = \max_{i\in [n]}\|\bA_i\|_{\psi_2}$.
  Then for every $x \geq 0$, with probability at least $1-2 e^{-x}$ one has
  \[
    \sqrt{d} - 2^{3.5} K^2 (\sqrt{\ln(9) n} + 6 \sqrt{x})
    \leq
    \smin(\bA)
    \leq
    \smax(\bA)
    \leq
    \sqrt{d} + 2^{3.5} K^2 (\sqrt{\ln(9) n} + 6 \sqrt{x})
  \]
\end{theorem}
Above two theorems lead to the following non-asymptotic version of a Bai-Yin law.
\begin{proof}[Proof of \cref{lem:non_asymptotic_bai_yin}]
  The proof considers two cases: 1) when number of observations exceeds the dimension, which is handled by the concentration of a minimal non-zero eigenvalue of a covariance matrix; 2) when dimension exceeds number of observations, which is handled by concentration of the Gram matrix.
  \paragraph{Case $n \geq d$.}
  We will apply~\cref{thm:matrix_concentration_indep_rows} with $\bA = (\bSigma^{\pinv\frac12} \bX)\tp$ whose rows are independent and isotropic, and in addition by Cauchy-Schwarz inequality:
  \begin{align*}
    \|\bSigma^{\pinv\frac12}\| \smin(\bX\tp)
    \geq
    \smin\pr{(\bSigma^{\pinv\frac12} \bX)\tp}
    \geq
    \sqrt{n} - 2^{3.5} K^2 (\sqrt{\ln(9) d} + \sqrt{x})
  \end{align*}
  with probability at least $1-e^{-x}$ for $x > 0$.
  Observing that $\|\bSigma^{\pinv\frac12}\| = \smin^+(\bSigma)^{-1/2}$, this implies that
  \begin{align*}
    &\smin(\bX\tp)
    \geq
      \sqrt{\smin^+(\bSigma)} \pr{\sqrt{n} - 2^{3.5} K^2 \pr{\sqrt{\ln(9) \, d} + \sqrt{x}}}~,
  \end{align*}
  while dividing through by $\sqrt{n}$, taking the non-negative part of the r.h.s.\, and squaring gives us
  \begin{align*}
    \lmin(\bhSigma)
    \geq
    \lmin^+(\bSigma) \pr{1 - 2^{3.5} K^2 \pr{\sqrt{\ln(9) \, \frac{d}{n}} + \sqrt{\frac{x}{n}}}}_+^2~.
  \end{align*}
  \paragraph{Case $n < d$.}
  In this case we essentially study concentration of a smallest singular value of a Gram matrix $\bhG = \frac{1}{d}\bX\tp\bX$.
  For the case $n < d$, \cref{thm:matrix_concentration_indep_rows} would give us a vacuous estimate, and therefore we rely on~\cref{thm:matrix_concentration_indep_cols} which requires additional assumption that columns of $\bX$ lie on a (elliptic) sphere of radius $\sqrt{d}$.
  In particular, similarly as before, applying \cref{thm:matrix_concentration_indep_cols} to the matrix $\bSigma^{\pinv\frac12} \bX$ with isotropic columns $\bSigma^{\pinv\frac12} \bX_i$ satisfying $\|\bSigma^{\pinv\frac12}\bX_i\| = \sqrt{d}$ \ a.s. for all $i \in [n]$, we get
  \begin{align*}
    \|\bSigma^{\pinv\frac12}\| \smin(\bX)
    \geq
    \smin\pr{\bSigma^{\pinv\frac12} \bX}
    \geq
    \sqrt{d} - 2^{3.5} K^2 (\sqrt{\ln(9) n} + 6 \sqrt{x})
  \end{align*}
  with probability at least $1-e^{-x}$ for $x > 0$.
  Again, this gives us
  \begin{align*}
    &\smin\pr{\bX}
    \geq
      \sqrt{\smin^+(\bSigma)} \pr{\sqrt{d} - 2^{3.5} K^2 \pr{\sqrt{\ln(9) \, n} + 6 \sqrt{x}}}~,
  \end{align*}
  while dividing through by $\sqrt{d}$, taking the non-negative part of the r.h.s.\, and squaring gives us
  \begin{align*}
    \lmin(\bhG)
    \geq
    \lmin^+(\bSigma) \pr{1 - 2^{3.5} K^2 \pr{\sqrt{\ln(9) \, \frac{n}{d}} + 6 \sqrt{\frac{x}{d}}}}_+^2~.
  \end{align*}

  Now we relate $\lmin(\bhG)$ to the smallest non-zero eigenvalue of $\bhSigma$ (see also~\citep[Remark 1]{bai1993limit}).
  The smallest eigenvalue of $d \bhG$ corresponds to $d-n+1$-th smallest eigenvalue of $n \bhSigma$, that is $d \lmin(\bhG) = n \lmin^+(\bhSigma)$.
  That said, multiplying the previous inequality through by $d/n$ and rearranging, we get
  \[
    \lmin^+(\bhSigma)
    \geq
    \lmin^+(\bSigma) \pr{\sqrt{\frac{d}{n}} - 2^{3.5} K^2 \pr{\sqrt{\ln(9)} + 6 \sqrt{\frac{x}{n}}}}_+^2
  \]
 The proof is now complete.
\end{proof}

\section{Bounding the third term (orthogonal complement)}
\label{app:sec:third_term}

Finally, we take care of $\E[\|\bwstar\|^2_{\bI-\bhatM}]$.
Clearly, in the underparameterized case $d \geq n$, $\E[\|\bwstar\|^2_{\bI-\bhatM}] = 0$ and so we will not consider such a case.
On the other hand, in the overparameterized case, we argue that whenever the spectrum of $\bhatM$ decays sufficiently quickly, the term of interest will behave as $\|\bwstar\|^2 / \sqrt{n}$.

Consider the following theorem due to \citep[Theorem 1]{shawe2005eigenspectrum}, which is concerned with the magnitude of projection onto partial eigenbasis of a covariance matrix (they state the theorem for Kernel-PCA, however we adapt it here for the Euclidean space):
\begin{theorem}[{\cite[Theorem 1]{shawe2005eigenspectrum}}]
  Denote the $k$-``tail'' of eigenvalues of $\bhSigma$ as
  \begin{align*}
    \hat{\lambda}^{> k} = \sum_{i=k+1}^n \hat{\lambda}_i~.
  \end{align*}
  Then, for any $\bz \in \reals^d$, with probability at least $1-\delta$ over $S$, for all $r \in [n]$,
\begin{align*}
  \E\br{\|P_{\bU_r}^{\perp}(\bz)\|_2^2}
  \leq
  \min_{k \in [r]}\cbr{\frac1n \cdot \hat{\lambda}^{>k} + \frac{1 + \sqrt{k}}{\sqrt{n}} \sqrt{\frac{2}{n} \sum_{i=1}^n \|\bX_i\|^2}}
  +
  \|\bz\|_2^2 \sqrt{\frac{18}{n} \cdot \ln\pr{\frac{2 n}{\delta}}}~.
\end{align*}
\end{theorem}
Since the rank of the covariance matrix in our case is $n$ and inputs lie on a unit sphere, we have that w.p.\ at least $1-\delta$ over $S$,
\begin{align}
\label{eq:orth_complement_bound}
  \E\br{\|\bwstar\|^2_{\bI - \bhatM}}
  \leq
  \min_{k \in [n]}\cbr{\frac1n \cdot \hat{\lambda}^{>k} + \pr{1 + \sqrt{k}} \sqrt{\frac{2}{n}}}
  +
  \|\bwstar\|_2^2 \sqrt{\frac{18}{n} \cdot \ln\pr{\frac{2 n}{\delta}}}
\end{align}
Thus, assuming that eigenvalues decay quickly enough that is $\hat{\lambda}_i = C b^{-i}$ for some constants $C > 0, b > 1, i \in \mathbb{N}$, the above projection behaves as (w.h.p. over $S$)
\begin{align*}
  \E\br{\|\bwstar\|^2_{\bI - \bhatM}} = \sctilO\pr{ \frac{\|\bwstar\|_2^2}{\sqrt{n}} } \quad \mathrm{as} \quad n \to \infty~.
\end{align*}
A natural question is whether we indeed typically observe a polynomial decay of the spectrum.
As an illustrative example we consider a simulation where inputs are sampled uniformly from a unit sphere for the sample size $n \in \cbr{2^i ~:~ i \in [13]}$ and $d=10 n$.
The $\min_{k \in [n]}\cbr{\dots}$ term in \cref{eq:orth_complement_bound} is plotted against the sample size in \cref{fig:poly_decay}.
We observe that the term exhibits polynomial decay.
\begin{figure}[H]
    \centering
    \includegraphics[width=8cm]{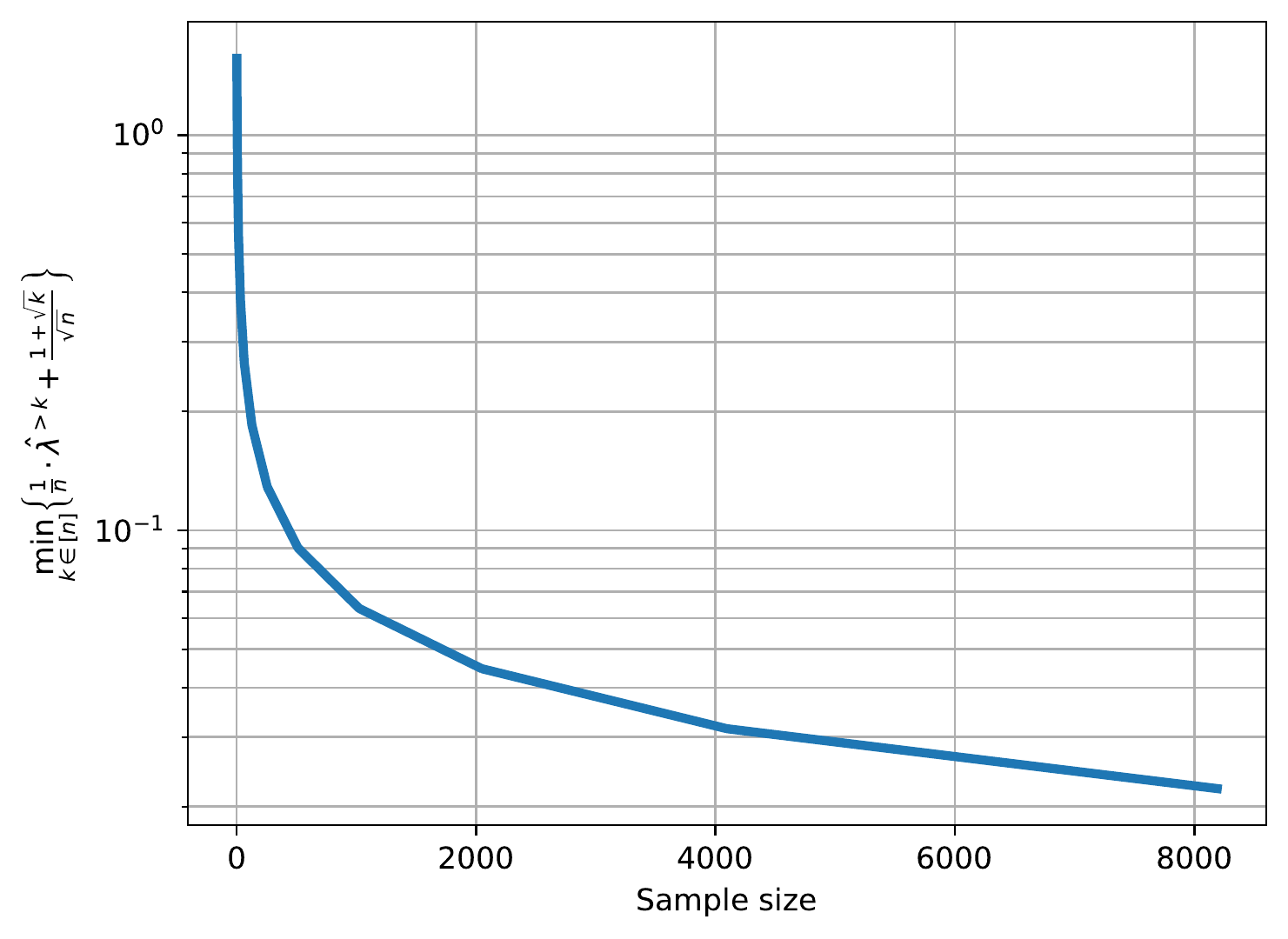}
    \caption{Decay of the $\min_{k \in [n]}\cbr{\dots}$ term in \cref{eq:orth_complement_bound} for inputs distributed on a unit sphere. Here $d = 10n$. Error bars are omitted due to insignificant scale.}
    \label{fig:poly_decay}
\end{figure}

\section{More on the effect of depth}\label{app:depth_exp}

In section~\ref{sec:empirical_discussion}, we suggested that the ill-conditioning of the intermediary features of a neural network is not only due to the size of the network, but also to the weights distribution across the layers. More particularly, we suggest here that the optimization difficulty we observe for deep neural networks is linked among other factors to the minimum eigenvalue of the activations of the penultimate layer.

To support our hypothesis, we run an experiment where we train networks of a fixed width (equal to 500) and depth varying from 2 to 10. We track the test error at various stages of training and the minimum eigenvalue of the features of the last layer. In Figure~\ref{fig:depth_err_eigval}, we can observe that as expected, the deeper the network, the harder it is to train them. This is reflected in the increasing test error. For the deepest network, simple gradient descent fails to obtain a reasonable performance even after 10000 iterations. Moreover, we observe that the deeper the network, the smaller is the minimum eigenvalue, and the most ill-conditioned settings get even worse with training.  

\begin{figure}[H]
  \centering
  \includegraphics[width=.95\textwidth]{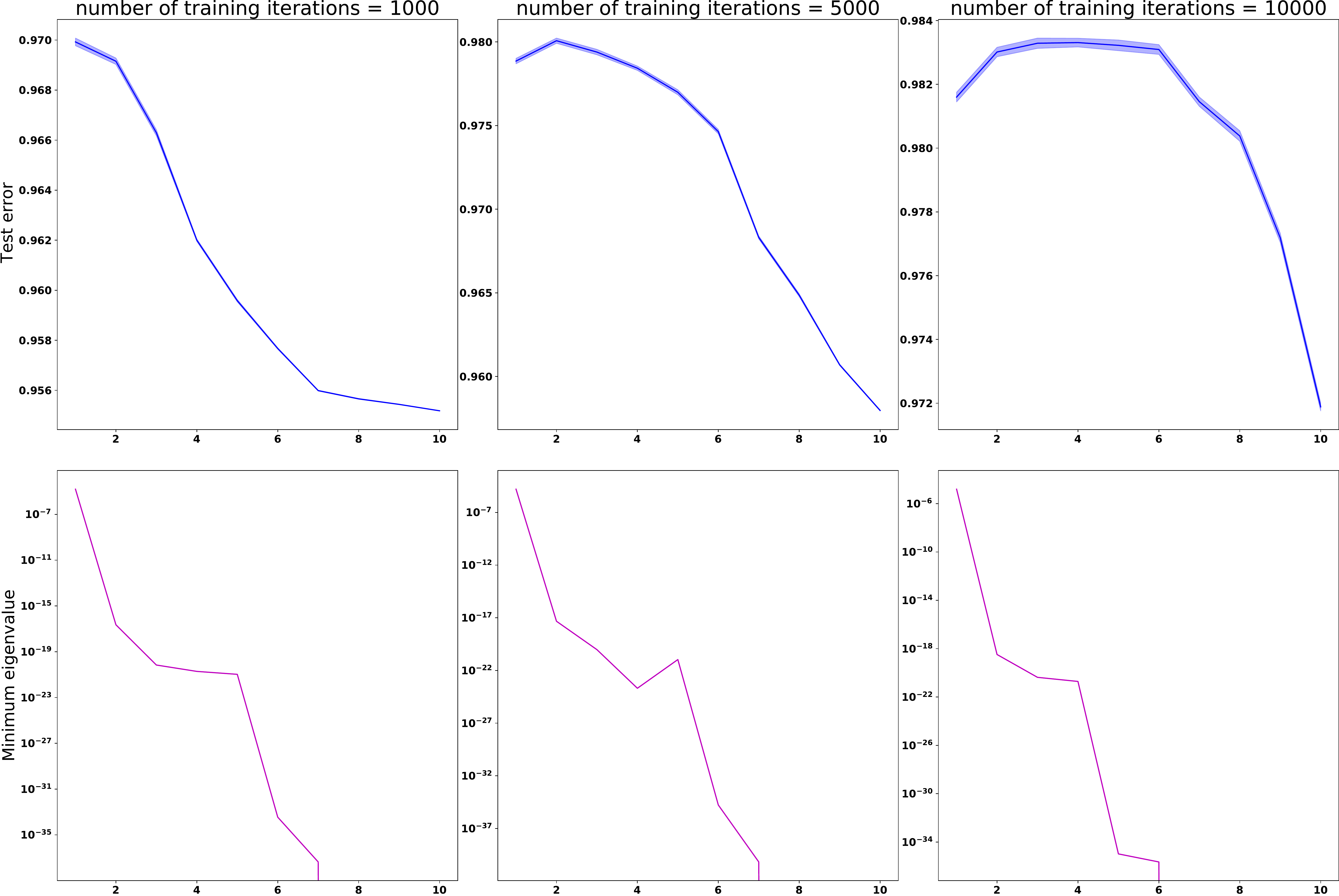}
  \caption{Mean test error and minimum eigenvalue for networks of fixed width = 500 and varying depth: MNIST - 1000 samples for training, 10000 samples for test, networks trained with gradient descent and step size 0.01.}
  \label{fig:depth_err_eigval}
\end{figure}

To further this analysis, we also compare networks with 3 hidden layers where we increase the width in all the layers and in the penultimate layer only, creating bottleneck in the earlier layers. This experiment complements Figure~\ref{fig:MNIST_depth}. In Figure~\ref{fig:MNIST_depth_bottleneck}, we observe that the bottleneck results in a more important drop in the eigenvalue around the width 1000 (width of the last layer in this case). Moreover, the minimum eigenvalue stays smaller than the other considered architectures when we increase the depth. This is reflected in a higher test error, confirming once more the effect of the conditioning of the last layer features on the final performance of the network when trained with gradient descent. 

\begin{figure}[H]
    \begin{minipage}[b]{.24\linewidth}
        \centering
        \includegraphics[align=c,width=\textwidth]{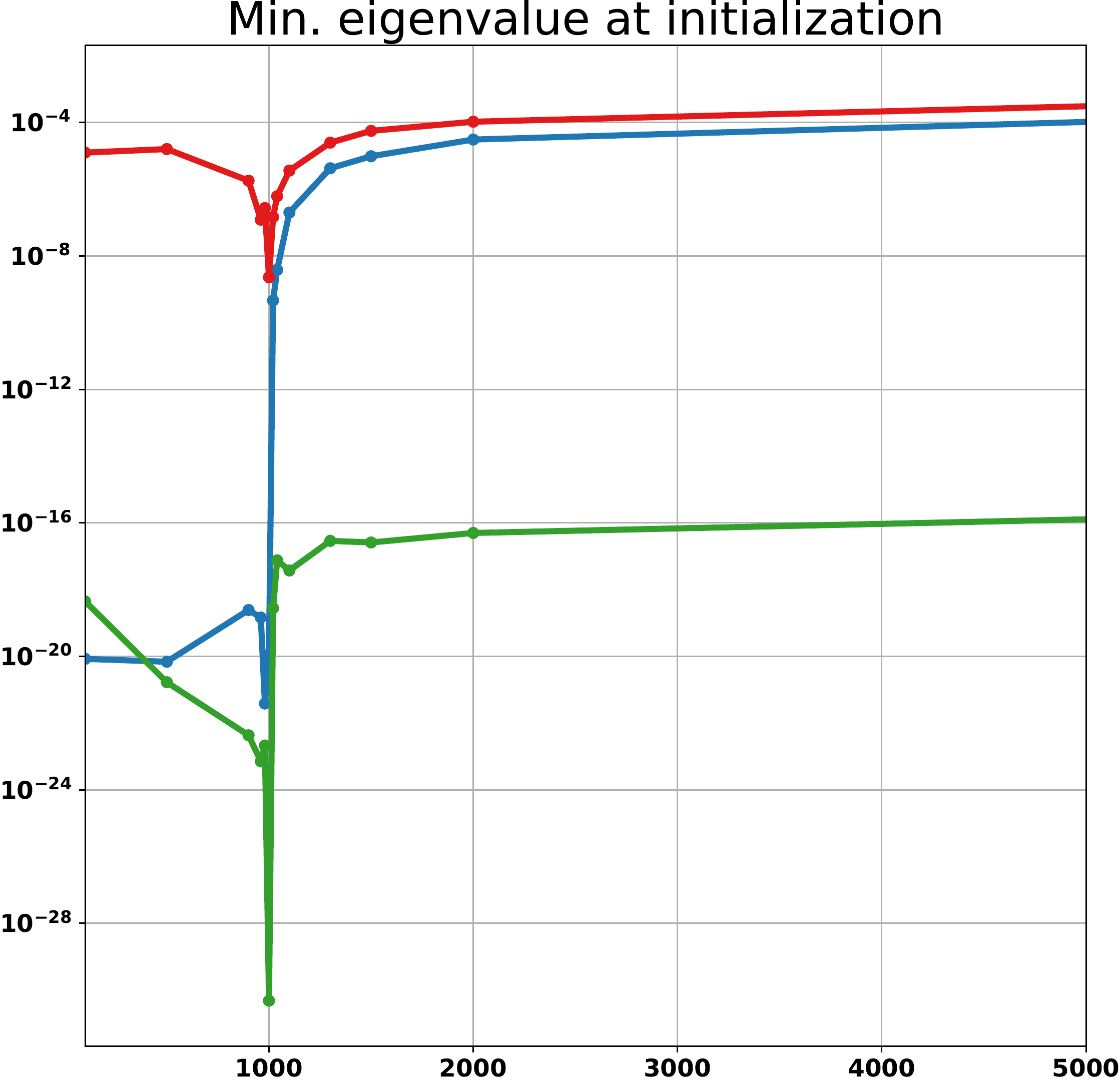}
        (a)
        \label{fig:MNIST_depth_eig_init}
        \end{minipage} \hfill
    \begin{minipage}[b]{.74\linewidth}
        \centering
        \includegraphics[align=c,width=\textwidth]{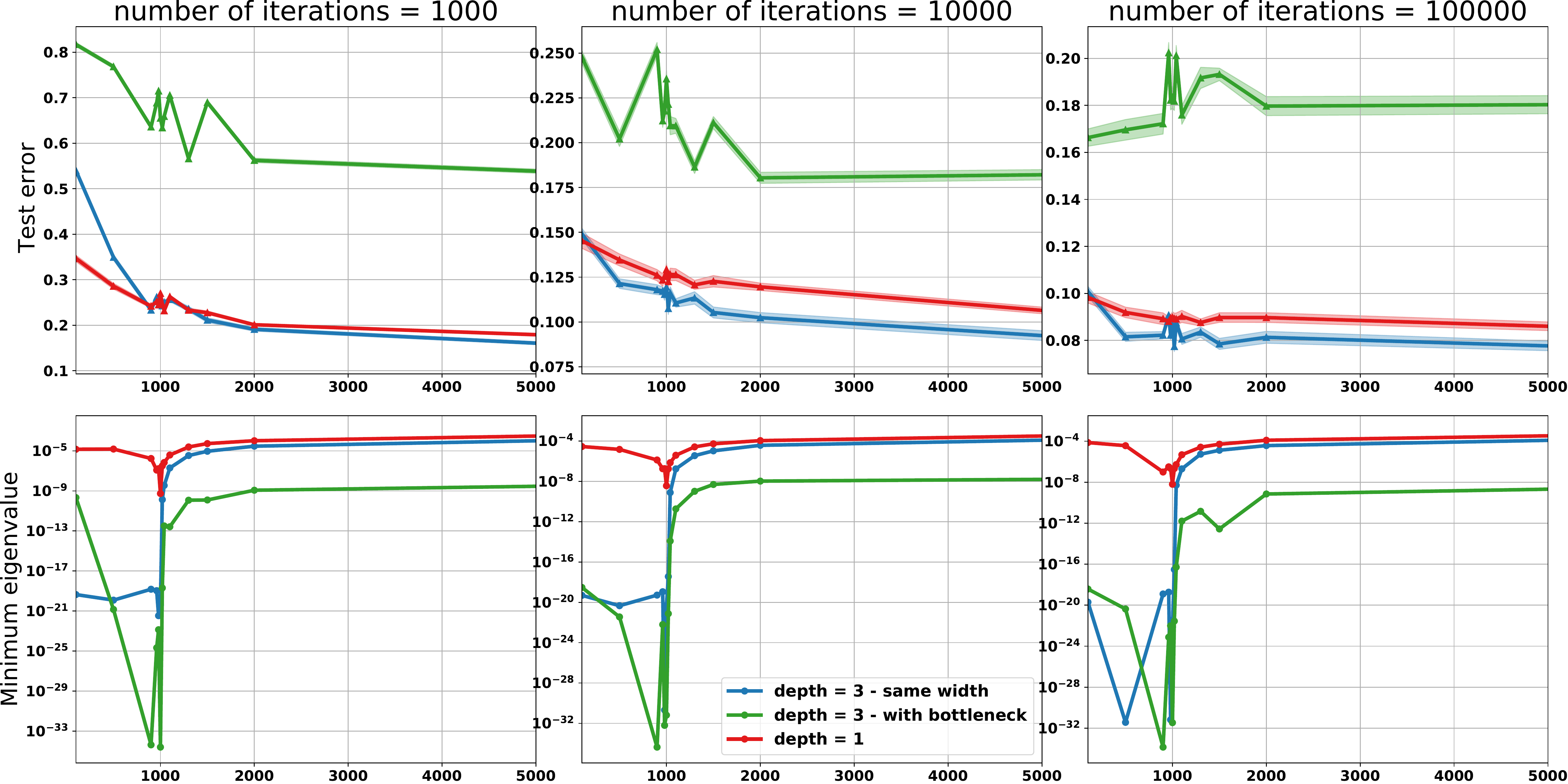}
        (b)
        \label{fig:MNIST_depth_error}
        \end{minipage}
    \caption{Training networks of increasing width with 1 and 3 hidden layers on MNIST - For the version with bottleneck, only the size of the last hidden layer is increased, while the other layers are composed of 10 neurons: (a) Minimum positive eigenvalue of the intermediary features at initialization - (b) Test error and corresponding minimum eigenvalue of the intermediary features at different iterations}\label{fig:MNIST_depth_bottleneck}
\end{figure}

\section{Additional Empirical Evaluation}
\label{sec:app:empirical}

\subsection{Experimental settings - More details}

In our experiments, we considered two datasets: MNIST and FashionMNIST. Both datasets have an input dimension of $784$ and a training set of $6.10^4$ samples. As our theory predicts that the drop in the minimum eigenvalue and the performance of the models happens when the feature size reaches the size of the training set, and in order to keep our model tractable, we use subsets of size 1000 of the training sets. These subsets are randomly chosen and kept the same when the size of the model increases. All the models are trained with plain gradient descent, with a fixed step size. We use a step size of $0.01$ unless stated otherwise. All the weights of the networks are initialized from a truncated normal distribution with a scaled variance. Finally, for the MNIST experiment in Figure~\ref{fig:MNIST_FashionMNIST}, the mean and standard errors are estimated from runs with different seeds. For the other experiments, the mean and standard errors of the test error are estimated by splitting the test set into 10 subsets. 

\subsection{More on the effect of architectural choices}
In section~\ref{sec:empirical_discussion}, we suggested that for neural networks the quantity of interest might also be $\lminh^+$ for intermediary features, which is affected by size of the model but also by the distribution of the weights and architectural choices. Section~\ref{app:depth_exp} shows some experiments that validate this hypothesis. To further our analysis, we question here the impact of skip connections on the eigenvalue of features at initialization. The difficulty that depth cause for the optimization of neural networks led to our reliance on skip connections among other tricks~\citep{De2020}. Here, we hypothesize that skip connections make the optimization of deep networks easier thanks to a better conditioning of the feature, through a less severe drop in the minimum eigenvalue around the interpolation threshold.  Figure~\ref{fig:skip_connection} shows that for a deep network with skip connection, the minimum eigenvalue of the penultimate layer activations behaves like this of a shallow neural network. 

\begin{figure}[H]
  \centering
  \includegraphics[width=.95\textwidth]{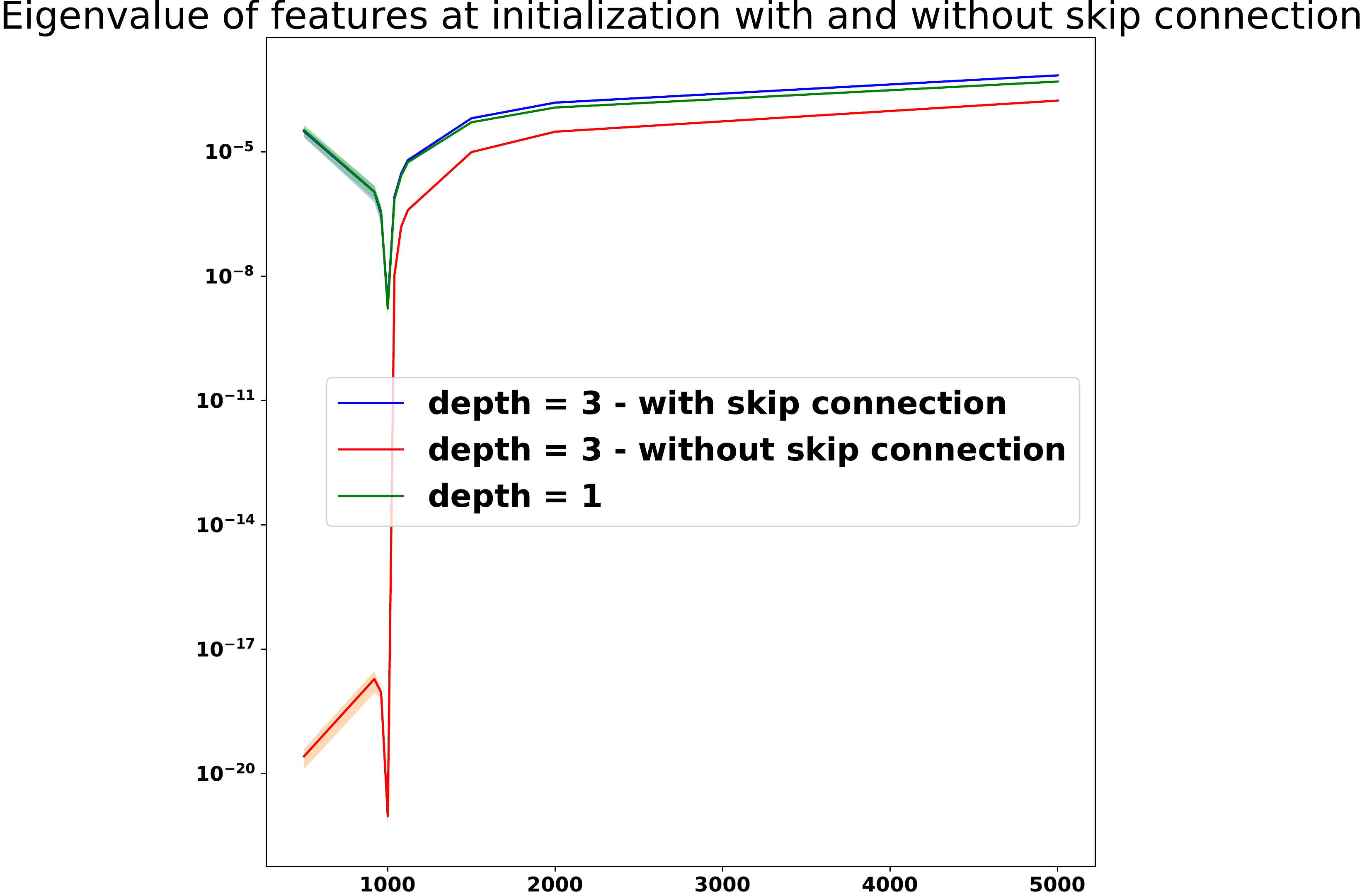}
  \caption{Mean minimum eigenvalue at initialization for networks of depths 1 and 3 and varying width. For the network of depth 3, we show two variants: with and without skip connection. The skip connection makes the deep network eigenvalue behave like the shallow network's.}
  \label{fig:skip_connection}
\end{figure}

\end{document}